\documentclass{article}

    \usepackage[final]{neurips_2023}

\usepackage[utf8]{inputenc} % allow utf-8 input
\usepackage[T1]{fontenc}    % use 8-bit T1 fonts
\usepackage{hyperref}       % hyperlinks
\usepackage{url}            % simple URL typesetting
\usepackage{colortbl, booktabs, array}       % professional-quality tables
\usepackage{amsfonts}       % blackboard math symbols
\usepackage{nicefrac}       % compact symbols for 1/2, etc.
\usepackage{microtype}      % microtypography
\usepackage{xcolor}         % colors
\usepackage{subcaption}
\usepackage{amsmath}
\usepackage{amssymb}
\usepackage{mathtools}
\usepackage{amsthm}
\usepackage{bm}
\usepackage{rotating}
\usepackage{algorithm, algorithmic}

\newcommand{\bostonpricesece}{0.003 & \textbf{0.0029} & \textbf{0.0029} & \textbf{0.0029} & 0.0056 & 0.041}
\newcommand{\bostonpricesinterval}{0.76 & 1.4 & 1.4 & 1.4 & \textbf{0.73} & 7.4}
\newcommand{\bostonpricesnll}{0.21 & 0.39 & 0.4 & 0.42 & \textbf{-0.24} & 1.6}
\newcommand{\bostonpricesstd}{\textbf{0.16} & 0.31 & 0.3 & 0.33 & 0.22 & 1.9}

\newcommand{\yachtece}{0.0044 & \textbf{0.0043} & 0.0044 & \textbf{0.0043} & 0.0081 & 0.039}
\newcommand{\yachtinterval}{0.76 & 2.3 & 2.3 & 2.3 & \textbf{0.3} & 11}
\newcommand{\yachtnll}{0.26 & 0.68 & 0.69 & 0.68 & \textbf{-2} & 2}
\newcommand{\yachtstd}{0.16 & 0.5 & 0.47 & 0.5 & \textbf{0.14} & 2.8}

\newcommand{\autompgece}{0.0036 & \textbf{0.0035} & \textbf{0.0035} & \textbf{0.0035} & 0.0053 & 0.044}
\newcommand{\autompginterval}{\textbf{0.6} & 1.7 & 1.8 & 1.7 & 0.96 & 11}
\newcommand{\autompgnll}{0.032 & 0.63 & 0.64 & 0.63 & \textbf{0.02} & 2}
\newcommand{\autompgstd}{\textbf{0.13} & 0.38 & 0.38 & 0.38 & 0.29 & 2.8}

\newcommand{\wineece}{\textbf{0.00047} & \textbf{0.00047} & \textbf{0.00047} & \textbf{0.00047} & 0.0067 & 0.0058}
\newcommand{\wineinterval}{\textbf{2.1} & 3.8 & 3.8 & 3.9 & 2.8 & 5.4}
\newcommand{\winenll}{1.2 & 1.3 & 1.3 & 1.3 & \textbf{-0.36} & 1.4}
\newcommand{\winestd}{\textbf{0.54} & 1 & 1 & 0.88 & 0.72 & 1.4}

\newcommand{\cementece}{{0.00071} & \textbf{0.00064} & \textbf{0.00064} & \textbf{0.00064} & 0.00076 & 0.032}
\newcommand{\cementinterval}{\textbf{0.93} & 2.5 & 2.5 & 2.5 & 2.1 & 11}
\newcommand{\cementnll}{\textbf{0.72} & 1.1 & 1.1 & 1.1 & 0.85 & 2}
\newcommand{\cementstd}{\textbf{0.25} & 0.64 & 0.64 & 0.64 & 0.57 & 2.8}

\newcommand{\kineightnmece}{\textbf{0.00016} & \textbf{0.00016} & \textbf{0.00016} & \textbf{0.00016} & 0.00053 & 0.028}
\newcommand{\kineightnminterval}{\textbf{0.26} & 0.47 & 0.47 & 0.48 & 0.44 & 1.6}
\newcommand{\kineightnmnll}{-0.54 & -0.65 & -0.65 & -0.63 & \textbf{-0.76} & 0.1}
\newcommand{\kineightnmstd}{\textbf{0.074} & 0.12 & 0.12 & 0.12 & 0.098 & 0.4}

\newcommand{\facebookece}{0.00044 & \textbf{0.00043} &\textbf{ 0.00043} & 0.00045 & 0.0089 & 0.044}
\newcommand{\facebookinterval}{\textbf{0.6} & 1.7 & 1.7 & 1.7 & 3.4 & 4.6}
\newcommand{\facebooknll}{3.6 & -1.3 & -1.3 & -1.2 & \textbf{-2.3} & 1.2}
\newcommand{\facebookstd}{\textbf{0.068} & 0.18 & 0.18 & 0.18 & 0.18 & 1.2}

\cellcolor{green!30}

\newcount\rowc

\makeatletter
\def\ttabular{%
\hbox\bgroup
\let\\\cr
\def\rulea{\ifnum\rowc=\@ne \hrule height 1.3pt \fi}
\def\ruleb{
\ifnum\rowc=1\hrule height 1.3pt \else
\ifnum\rowc=6\hrule height \heavyrulewidth 
   \else \hrule height \lightrulewidth\fi\fi}
\valign\bgroup
\global\rowc\@ne
\rulea
\hbox to 7em{\strut \hfill##\hfill}%
\ruleb
&&%
\global\advance\rowc\@ne
\hbox to 7em{\strut\hfill##\hfill}%
\cr}
\def\endttabular{%
\crcr\egroup\egroup}

\usepackage[capitalize,noabbrev]{cleveref}

\theoremstyle{plain}
\newtheorem{theorem}{Theorem}[section]

\newtheorem{lemma}[theorem]{Lemma}

\theoremstyle{definition}

\newtheorem{assumption}[theorem]{Assumption}
\theoremstyle{remark}
\newtheorem{remark}[theorem]{Remark}
\newcommand{\x}{\bm{x}}
\newcommand{\y}{\bm{y}}
\newcommand{\X}{\mathcal{X}}

\newcommand{\D}{\mathcal{D}}
\newcommand{\Dtr}{\D_{\text{tr}}}
\newcommand{\Dcal}{\D_{\text{cal}}}

\title{Sharp Calibrated Gaussian Processes}

\author{%
  Alexandre Capone\\%\thanks{Use footnote for providing further information about author (webpage, alternative address)---\emph{not} for acknowledging funding agencies.} \\
  Technical University of Munich\\
  \texttt{alexandre.capone@tum.de} \\
  \And
  Sandra Hirche \\
  Technical University of Munich\\
  \texttt{hirc@cit.tum.de} \\
  \AND
  Geoff Pleiss \\
  University of British Columbia \\
  Vector Institute\\
  \texttt{geoff.pleiss@stat.ubc.ca} \\
}

\begin{document}

\maketitle

\begin{abstract}
While Gaussian processes are a mainstay for various engineering and scientific applications, the uncertainty estimates don't satisfy frequentist guarantees and can be miscalibrated in practice. State-of-the-art approaches for designing calibrated models rely on inflating the Gaussian process posterior variance, which yields confidence intervals that are potentially too coarse. To remedy this, we present a calibration approach that generates predictive quantiles using a computation inspired by the vanilla Gaussian process posterior variance but using a different set of hyperparameters chosen to satisfy an empirical calibration constraint. This results in a calibration approach that is considerably more flexible than existing approaches, which we optimize to yield tight predictive quantiles. Our approach is shown to yield a calibrated model under reasonable assumptions. Furthermore, it outperforms existing approaches in sharpness when employed for calibrated regression.
\end{abstract}

\section{Introduction}
\label{section:introduction}

Gaussian process (GP) regression offers an ambitious proposition: by conditioning a model on measurement data, we are provided with a Gaussian probability distribution for the unseen data. Assuming that the posterior probability distribution holds, we can then directly calibrate our model using the inverse error function. Though the distribution of unseen data seldom follows the Gaussian prior distribution, and the GP generally does not adapt adequately to the observed distributions after being conditioned on the data, GPs have become one of the most powerful and established regression techniques. Besides having found widespread use in machine learning \citep{deisenroth2015gaussian,Srinivas2012}, their good generalization properties have motivated applications in the fields of control \citep{Kocijan2016}, astrophysics \citep{roberts2013gaussian} and chemistry \citep{deringer2021gaussian}, to name a few. Furthermore, the Bayesian paradigm offers a powerful tool to analyze the theoretical properties of different regression techniques \citep{Srinivas2012,capone2022gaussian}.

In this paper, we present a novel approach to obtaining sharp calibrated Gaussian processes, i.e., Gaussian processes that provide concentrated predictive distributions that accurately match the observed data. Instead of computing confidence intervals by inflating the Gaussian process posterior variance, our approach discards it and computes a new quantity inspired by the computation  of the posterior variance, where all hyperparameters are chosen in a way that results in both accurate and sharp calibration. In other words, we train two separate Gaussian processes: one for the predictive mean and one for obtaining predictive quantiles, which is exclusively used for calibration purposes. By doing so, we reach considerably more flexibility than existing calibration approaches, which enables us to additionally optimize the sharpness of the calibration. Our approach outperforms several state-of-the-art calibration approaches in terms of sharpness while still yielding similar calibration performance. Furthermore, it is competitive compared to a neural network-based method in sharpness without sacrificing calibration performance.

\paragraph{Notation.} We use $\mathbb{R}_+$ to denote the non-negative real numbers. Boldface lowercase/uppercase characters denote vectors/matrices. For two vectors $\bm{a}$ and $\bm{a}'$ in $\mathbb{R}^d$, we employ the notation $\bm{a}\leq \bm{a}'$ to denote componentwise inequality, i.e., ${a}_i\leq {a}'_i$, $i=1,\ldots,d$. For a square matrix $\bm{K}$, we use $\vert\bm{K}\vert$ to denote its determinant, and $[\bm{K}]_{ij}$ to denote the entry corresponding to the $i$-th row and $j$-th column.

\section{Related Work}
\label{section:relatedwork}

\paragraph{Calibration of Classification Models.}
There has been extensive work on obtaining calibrated models in the domain of classification. While there are many methods that do not employ post-processing, we only focus here on methods that employ some form of post-processing. Most forms of post-processing-based calibration for classification fall into the category of conformal methods \citep{vovk2005algorithmic}, which, given an input, aim to produce sets of labels that contain the true label with a pre-specified probability. Arguably the two most common forms of calibration are isotonic regression \citep{niculescu2005predicting} and Platt scaling \citep{platt1999probabilistic}. In \citet{niculescu2005predicting}, Platt scaling and isotonic regression are analyzed extensively for different types of predictive models. In \citet{guo2017calibration}, a modified form of Platt scaling for modern classification neural networks is proposed. 

\paragraph{Calibration of Regression Models.} Though initially developed for classification, conformal calibration has been extended to regression settings.
In \citet{Lakshminarayanan2017simple}, a calibration approach was proposed for deep ensembles. \citet{gal2017concrete} propose a dropout-based technique for calibrating deep neural networks. However, these approaches require changing the regressor, potentially deteriorating its predictive performance. It should be noted that Bayesian neural networks \citep{mackay1995bayesian}, while being able to provide credible sets for the output, fully trust the posterior, resulting in a naively calibrated model that seldom reflects the data's distribution. As a remedy for this, \citet{kuleshov2018accurate} present a recalibration approach that scales a model's predictive quantiles to satisfy the observed data's distribution. In \citet{vovk2020conformal}, a similar approach is presented, where interpolation between scaling factors is randomized, and a theoretical analysis is provided.
An extension of both \citet{kuleshov2018accurate} and \citet{vovk2020conformal} and other recalibration is proposed in \citet{marx2022modular}, along with corresponding theoretical guarantees. While these methods have been shown to yield well-calibrated models, the resulting predictive quantiles are potentially much too crude, resulting in predictions that perform poorly in terms of sharpness, i.e., the corresponding confidence intervals will overestimate the model error by a very large margin. To remedy this,  \citet{song2019distribution} and \citet{kuleshov2022calibrated} propose optimizing the parameters of a recalibration model by obtaining calibration on a distribution level. However, while \citet{song2019distribution} relies on complex approximations and provides no theoretical guarantees, \citet{kuleshov2022calibrated} does not allow to optimize for sharpness directly, and calibration is only guaranteed asymptotically as the number of data grows.

\section{Problem Statement}
\label{section:problemstatement}
Consider a compact input space $\mathcal{X}\subset \mathbb{R}^d$, and output space $\mathcal{Y}\subset \mathbb{R}$, and an unknown data distribution $\Pi$ on $\mathcal{X}\times \mathcal{Y}$. Consider a model conditioned on training data $\Dtr\sim \Pi$ that, for every $\bm{x} \in\mathcal{X}$ and confidence level $\delta \in [0,1]$, returns a base prediction $\mu_{\Dtr}(\bm{x})$ and an additive cut-point term $\beta_{\delta}\sigma_{\Dtr}(\delta,\bm{x})$, where $\sigma_{\Dtr}(\delta,\bm{x})\geq 0$ and $\beta_{\delta}\in \mathbb{R}$ is potentially negative, such that $\mu_{\Dtr}(\bm{x})+\beta_{\delta}\sigma_{\Dtr}(\delta,\bm{x})$ corresponds to a predictive $\delta$-quantile. The model is then said to be calibrated if
\begin{align}
    \mathbb{P}_{\bm{x},y \sim \Pi}\Big( y - \mu_{\Dtr}(\bm{x}) \leq \beta_{\delta}\sigma_{\mathcal{D}}(\delta,\bm{x})\Big) = \delta
\end{align}
holds for every $\delta \in [0,1]$. Furthermore, the calibrated model is also said to be sharp if the corresponding predictive distributions are concentrated \citep{gneiting2007probabilistic}, i.e., if the centered confidence intervals induced by the predictive quantiles
\begin{align} \vert \beta_{\delta}
\label{eq:centeredinterval}\sigma_{\Dtr}(\delta,\bm{x}) - \beta_{1-\delta}\sigma_{\Dtr}(1-\delta,\bm{x}) \vert \end{align}
are as small as possible for every $\delta \in [0,1]$. %This corresponds t \[\mathbb{E}_{\bm{x},y \sim \Pi}\left(\vert\beta_{\delta}\sigma_{\mathcal{D}}(\delta,\bm{x})\vert\right)\] is small for every $\delta$ \citep{kuleshov2018accurate}. 
Our goal is to find a sharply calibrated model based on GP regression.

\section{Gaussian Process Regression}
\label{section:gpregression}
In this section, we briefly review GP regression, with particular focus on the choice and influence of hyperparameters.

A GP is formally defined as a collection of random variables, any finite subset of which is jointly Gaussian \citep{Rasmussen2006}. It is fully specified by a prior mean function, which we set to zero without loss of generality, and a hyperparameter-dependent covariance function, called kernel $k:\bm{\varTheta} \times \X \times \X \rightarrow \mathbb{R}$, where $\bm{\varTheta}$ denotes the hyperparameter space. The core concept behind GP regression lies in assuming that any finite number of measurements of an unknown function $f:\X\rightarrow\mathbb{R}$ at arbitrary inputs $\bm{x}_1,\ldots,\bm{x}_N\in\X$ are jointly Gaussian with mean zero and covariance $\bm{K}(\bm{\theta})$, where 
\[
[\bm{K}(\bm{\theta})]_{ij} = k(\bm{\theta},\bm{x}_i, \bm{x}_j) 
\]
consists of kernel evaluations at the test inputs.

Given a set of noisy observations $\D = \{\x_i, y_i\}_{i=1}^N$, where $y_i \coloneqq f(\x_i) + \varepsilon_i$ and $\varepsilon_i\sim\mathcal{N}(0,\sigma_0^2)$ is iid Gaussian measurement noise, we can condition the GP on them to obtain the posterior distribution. The posterior distribution for a new observation $y^*$ at an arbitrary input $\bm{x}^*$ is again Gaussian distributed, with mean and variance
\begin{subequations}
    \begin{align}
        \mu_{\Dtr}(\bm{\theta},\bm{x}^*) =& \bm{k}(\bm{\theta})\left(\bm{K}(\bm{\theta}) + \sigma_0^2\bm{I} \right)^{-1}\bm{y},\\
        \begin{split}
        \sigma^2_{\Dtr}(\bm{\theta},\bm{x}^*)
        =&{k}(\bm{\theta},\bm{x}^*,\bm{x}^*) -  \bm{k}(\bm{\theta})\left(\bm{K}(\bm{\theta}) + \sigma_0^2\bm{I} \right)^{-1}\bm{k}(\bm{\theta}) + \sigma_0^2,
        \end{split}
    \end{align}
\end{subequations}
with $\bm{k}(\bm{\theta})\coloneqq \left({k}(\bm{\theta},\bm{x}^*,\bm{x}_1), \ldots, {k}(\bm{\theta},\bm{x}^*,\bm{x}_N)\right)$, $\bm{y}=\left(y_1,\ldots,y_N\right)$, and $\bm{I}$ denoting the identity matrix. In this paper, we restrict ourselves to kernels $k(\cdot,\cdot,\cdot)$ that yield a posterior variance that is monotonically increasing with respect to the hyperparameters $\bm{\theta}$, as specified in the following assumption.
\begin{assumption}
\label{assumption:montonicity}
    The posterior variance $\sigma_{\Dtr}^2(\bm{\theta},\bm{x}^*)$ is a continuous function of $\bm{\theta}$. Furthermore, for all hyperparameters $\bm{\theta},\bm{\theta}'\in\bm{\varTheta}$ with $\bm{\theta}\leq \bm{\theta}'$, it holds that $\sigma_{\Dtr}^2(\bm{\theta},\bm{x}^*)\leq \sigma_{\Dtr}^2(\bm{\theta}',\bm{x}^*)$.
\end{assumption}

\Cref{assumption:montonicity} holds trivially for the signal variance of a kernel. Moreover, it holds for any hyperparameters that lead to a monotonous increase in the Fourier transform, which is the case, e.g., for the inverse lengtschale of stationary kernels up to a multiplicative factor corresponding to the ratio of the lengthscales  \citep{capone2022gaussian}. Furthermore, several results that employ the so-called fill-distance indicate that \Cref{assumption:montonicity} holds for the inverse lengthscale of a broad class of stationary kernels as the lengthscale becomes very large or very small \citep{Wendland2004}. In our experiments, we observed that \Cref{assumption:montonicity} was never violated for the inverse lengthscale of the squared-exponential kernel. \Cref{assumption:montonicity} will be leveraged to define a cumulative density function by also changing the hyperparameters corresponding to the posterior covariance, as opposed to simply scaling it.
% (see, e.g., Chapter 11 in Scattered Data Approximation by Holger Wendland, 2011). 

%\Cref{assumption:montonicity} is not very restrictive, as it holds for many commonly used kernels, e.g., linear or isotropic kernels \citep{capone2022gaussian}. Furthermore, it can also be enforced  via a reparametrization of the kernel, e.g., if the hyperparameters correspond to lengthscales, then we choose $\bm{\theta}$ to be the inverse lengthscales. 

%Henceforth, we will also refer to the posterior mean $\mu_{\Dtr}(\bm{\theta},\bm{x}^*)$ as the model.

Arguably one of the most challenging aspects of GP regression lies in the choice of hyperparameters $\bm{\theta}$, as they ultimately determine various characteristics of the posterior, e.g., smoothness and amplitude. In practice, the most common way of choosing $\bm{\theta}$ is by maximizing the log marginal likelihood
\begin{align}
	\label{eq:loglikelihood}
	\begin{split}
		&\log p(\bm{y} \vert \bm{X}, \bm{\theta}) = -\frac{1}{2}\log\vert \bm{K}(\bm{\theta}) + \sigma_0^2 \bm{I} \vert -\frac{N}{2}\log(2\pi) - \frac{1}{2}\y^\intercal \left(\bm{K}(\bm{\theta}) +\sigma_0^2 \bm{I}\right)^{-1}\y .
	\end{split}
\end{align}
%where $\vert\cdot\vert$ denotes the determinant operator. 

\begin{figure*}
\centering
\includegraphics[width =0.7\columnwidth]{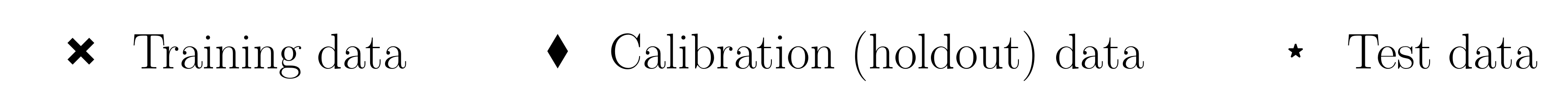}
\begin{subfigure}[b]{0.47\textwidth}
\includegraphics[height = 0.56\columnwidth, trim={1.5cm 1cm 0cm 0},clip]{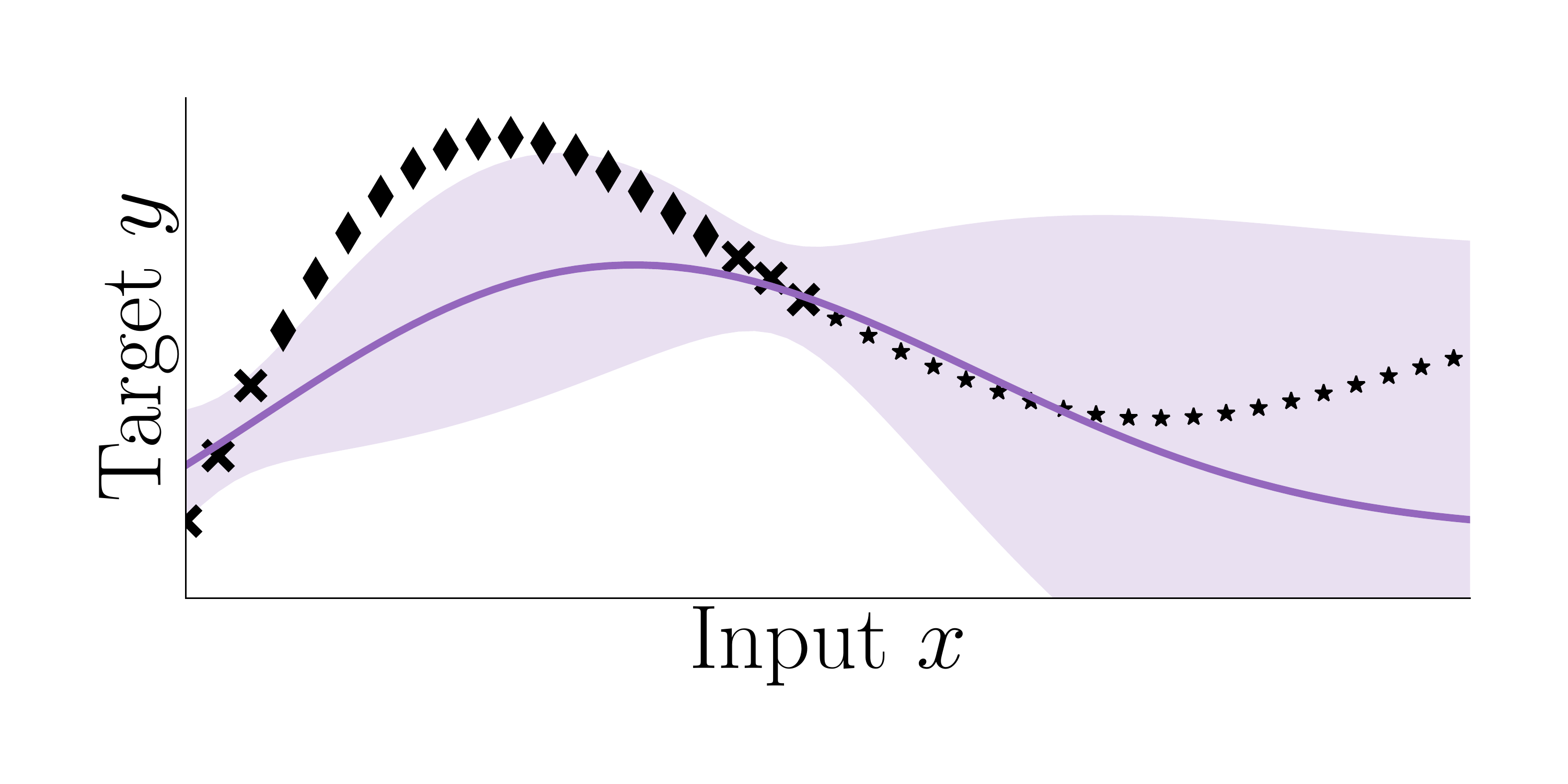}
\caption{Bayesian approach.} 
\end{subfigure}
\hfill
\begin{subfigure}[b]{0.47\textwidth}
\includegraphics[height = 0.56\columnwidth, trim={3.5cm 1cm 0 0},clip]{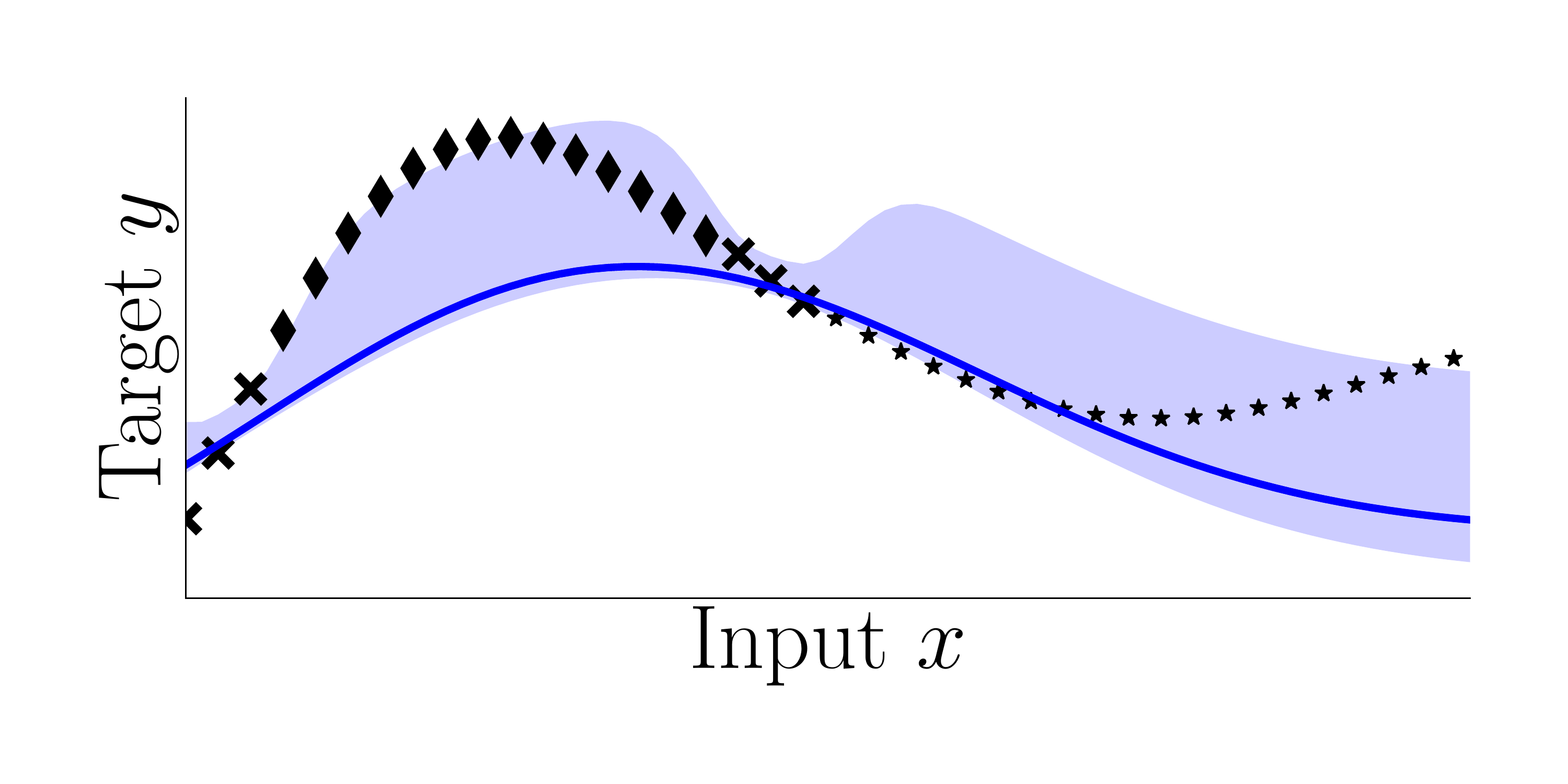}
\caption{Our approach.}
\end{subfigure}
\caption{Confidence interval of $95\%$ (shaded regions) obtained with purely Bayesian approach (a), where the inverse error function is employed to compute $\beta_{\delta}$, and our approach (b). Solid lines represent the predictive mean. The confidence interval obtained with the Bayesian approach is not only grossly overconfident (contains less than $80\%$ of the total data) but also partially extremely loose, exhibiting unnecessarily large confidence intervals far away from the data. By contrast, our approach is both accurate and tight.}
\label{fig:poorfit}
\end{figure*}

% \begin{figure}
% \centering
% \includegraphics[width =\columnwidth]{Sinusoidal_Bayes.pdf}
% \caption{Confidence interval of $90\%$ obtained with purely Bayesian approach, where the inverse error function is employed to compute $\beta_{\delta}$. The confidence region does not capture the training nor the test data with $90\%$ accuracy, indicating that it is exceedingly confident.}
% \label{fig:poorfit}
% \end{figure}

In terms of \textit{posterior mean} quality, i.e., predictive performance of $\mu_{\Dtr}(\bm{\theta},\bm{x}^*)$, choosing the hyperparameters in this manner is often the most promising option, since it seeks a trade-off between model complexity and data fit, and has repeatedly been shown to yield a satisfactory mean square error when applied to test data \citep{Rasmussen2006}. However, even when employing log-likelihood maximization, the posterior GP distribution is seldom well calibrated, and, in practice, the data often has a significantly different distribution. As a result, quantiles obtained with a purely Bayesian approach are either too confident or too conservative in practice \citep{capone2022gaussian,fong2020marginal}. Furthermore, the restrictions imposed by the GP prior often produce a posterior variance $\sigma^2(\bm{\theta},\bm{x}^*)$ that is grossly conservative. See \Cref{fig:poorfit} for an illustration.

% generally cannot be said of the posterior variance $\sigma^2(\bm{\theta},\bm{x}^*)$, particularly whenever it is employed to obtain predictive quantiles, as it is either too confident or too conservative in practice \citep{capone2022gaussian,fong2020marginal}. This lack of accuracy stems from the minimization of the log marginal likelihood \eqref{eq:loglikelihood}, which measures the GP's ability to explain the existing data, but not its ability to make predictions on unseen data \citep{lotfi2022bayesian}. This is because, on the one hand, the log marginal likelihood \eqref{eq:loglikelihood} does not account for calibration, i.e., how well the posterior variance explains the model error. On the other hand, in practice the Gaussian prior for the data often deviates much too strongly from its true distribution, and the posterior, being also Gaussian, typically also presents a poor fit. This discrepancy between the posterior and the unseen data has serious implications for whenever we aim to obtain a calibrated model, i.e., a model that provides accurate confidence intervals for the unseen data. This is particularly true of the Bayesian approach to confidence intervals, where we obtain confidence intervals by simply multiplying the posterior standard deviation with the corresponding z-score. See \Cref{fig:poorfit} for an illustration.

\section{Proposed Approach}
\label{section:proposedapproach}

In this section, we present our approach to obtaining calibrated GPs by discarding the posterior variance and obtaining alternative predictive quantiles using a quantity inspired by the posterior variance and new hyperparmeters. As we will demonstrate, our approach has numerous advantages. First, by varying the hyperparameters used to obtain predictive quantiles, the resulting quantiles are sharper than what can be obtained simply by multiplying the posterior variance with an appropriate constant. Secondly, by exploiting the monotonicity of the hyperparameters, our approach can be used to obtain intervals for multiple confidence levels $\delta$ in a very efficient manner. Finally, our method is backed by tight theoretical guarantees, obtained by exploiting its connection to conformal prediction.

\subsection{Sharp Calibrated GP for Single Confidence Level $\delta$}
\label{subsection:singledelta}

In the following, we describe how to obtain a sharply calibrated GP for a fixed desired calibration level $\delta$. %In this paper, we quantify sharpness using the measure suggested by \citet{kuleshov2018accurate}, whereby a model is said be sharp if the average posterior variance of the calibrated model is low{\color{red} This is inaccurate. It might be best to establish some form of equivalence between sharpness metrics first.}. 
Instead of scaling the GP posterior variance to meet the desired calibration level $\delta$, we propose computing predictive quantiles by training a new quantity similar to the posterior variance but with new hyperparameters. In other words, we discard the posterior variance of the first GP and replace it by a quantity that corresponds to the posterior variance of a different GP, which we train separately. This way, we allow for more degrees of freedom during calibration. We then leverage this additional freedom to minimize the distance of the quantiles to the predictive mean, which yields a sharply calibrated model.

We assume to have a data set $\D$, which we split into training data $\Dtr$ and calibration (holdout) data $\Dcal$, with $\D = \Dtr \cup \Dcal$. The training data $\Dtr$ will be used to compute the posterior, whereas $\Dcal$ will be used to calibrate the model. Note that while not splitting the data might be reasonable when using other types of regression models, e.g., Bayesian or ensemble neural networks, splitting the data in the case of GPs is beneficial for providing accurate quantiles for data out of distribution. This is because, for many commonly used kernels, the GP posterior distribution is considerably more concentrated for test points close to the training data $\Dtr$ than for those far away from $\Dtr$. Since we wish to obtain a model that is calibrated for data both close and far away from the training data, we take this into account during training by splitting the data. We also assume to have a \textit{predictive} mean $\mu_{\Dtr}(\bm{\theta}^{R},\cdot)$, corresponding to GP posterior mean function, where the regressor hyperparameters $\bm{\theta}^{R}$ were obtained, e.g., via log-likelihood maximization, as discussed in \Cref{section:gpregression}. However, note that any other way of choosing the posterior mean hyperparameters $\bm{\theta}^{R}$ is permitted. %, e.g., expert knowledge, cross-validation, or by minimizing metrics that take generalization into account \citep{lotfi2022bayesian}. 

Given $\mu_{\Dtr}(\bm{\theta}^{R},\cdot)$ and $\D = \Dtr \cup \Dcal$, we then follow the convention of other recalibration approaches \citep{kuleshov2018accurate,marx2022modular} and aim to obtain, for an arbitrary $0\leq\delta\leq1$, a scalar $\beta_\delta$ and vector of hyperparameters $\bm{\theta}_{\delta}$, such that the corresponding predictive quantile $\mu_{\Dtr}(\bm{\theta}^{R},\cdot)+\beta_\delta\sigma_{\Dtr}(\bm{\theta}_{\delta},\cdot)$ contains $\delta$ times the total amount of data points. For this reason, we henceforth refer to $\bm{\theta}_{\delta}$ as \textit{calibration hyperparameters}.

% A naive way of choosing $\beta_{\delta}$ and $\bm{\theta}_{\delta}$ would be by guaranteeing that 
% \[\Big\{\bm{x} \times y \in \mathcal{X} \times \mathcal{Y} \ \vert \ y - \mu_{\Dtr}(\bm{\theta}^{R},\bm{x}) \leq \beta_{\delta}\sigma_{\Dtr}(\bm{\theta}_{\delta},\bm{x}) \Big\}\] contains exactly $\delta$ times the desired number of data from $\D$. While a similar approach might be reasonable, e.g., for Bayesian or ensemble neural networks, it should be avoided in the case of GPs. This is because $\sigma_{\Dtr}(\bm{\theta}_{\delta},\cdot)$ is approximately proportional to the measurement noise at $\D$, whereas it behaves considerably differently for test points that are not close to the data $\D$. Since we wish to obtain a model that is calibrated within the whole input space $\X$, we have to take this into account during training. To this end, we randomly split the total available data $\D$ into posterior variance training data $\Dtr$ and calibration (holdout) data $\Dcal$, with $\D = \Dtr \cup \Dcal$. The posterior mean $\mu_{\Dtr}(\bm{\theta}^R,\cdot)$ and the predictive quantile $\beta_\delta\sigma_{\Dtr}(\bm{\theta}_{\delta},\cdot)$ are then conditioned uniquely on $\Dtr$, and $\Dcal$ is employed to choose sharp calibration hyperparameters $\bm{\theta}_{\delta}$. 
In order to obtain a sharply calibrated GP model, ideally we would like to choose $\beta_\delta$ and $\bm{\theta}_\delta$ such that they minimize the expected length of the centered intervals \eqref{eq:centeredinterval} subject to calibration. However, this optimization problem is hard to solve. Hence, we instead attempt to improve model sharpness by concentrating the predictive distribution around the predictive mean $\mu_{\Dtr}(\bm{\theta}^{R},\cdot)$, i.e., by minimizing the deviation of the quantiles from $\mu_{\Dtr}(\bm{\theta}^{R},\cdot)$. This corresponds to solving the optimization problem
\begin{align}
\label{eq:desiredequationcal}
    \begin{split}
        \min_{\substack{\beta_{\delta} \in\mathbb{R}\,\\  \bm{\theta}_\delta \in \bm{\varTheta}}} \ &  \sum^{N_{\text{cal}}}_{i=1}   \beta_{\delta}^2 \sigma_{\Dtr}^2\left(\bm{\theta}_{\delta}, \bm{x}_{\text{cal}}^i\right) \\
        \text{s.t.} \ &\sum^{N_{\text{cal}}}_{i=1} \frac{\mathbb{I}_{\geq 0}\left( \Delta y_{\text{cal}}^i  -  \beta_{\delta}\sigma_{\Dtr}\left(\bm{\theta}_{\delta}, \bm{x}_{\text{cal}}^i\right)\right)}{N_{\text{cal}}+1} = \delta
    \end{split}
\end{align}
where $\left\{\bm{x}_{\text{cal}}^i, y_{\text{cal}}^i\right\}\in \Dcal $ are samples from the calibration data set, $\Delta y_{\text{cal}}^i\coloneqq y_{\text{cal}}^i - \mu_{\Dtr}(\bm{\theta}^R,\bm{x}_{\text{cal}}^i)$ corresponds to the difference between predicted and measured output, and $N_{\text{cal}} = \vert \Dcal\vert $ to the number of data points used for calibration.
\begin{remark}
    Note that we employ $N_{\text{cal}}+1$ in the denominator in \eqref{eq:desiredequationcal} instead of $N_{\text{cal}}$. Though this choice makes little difference in practice, we require it for theoretical guarantees.
\end{remark}

The equality constraint in \eqref{eq:desiredequationcal} is generally infeasible, e.g., if  $N_{\text{cal}}=2$ and $\delta = 0.5$, and is discontinuous, making it hard to solve. As it turns out, this can be easily remedied without any detriment to sharpness or calibration, and the problem can be rendered considerably easier to solve by substituting the equality constraint in \eqref{eq:desiredequationcal} with
\begin{align}
\label{eq:betaasqlin}
    \beta_{\delta} = q_{\text{lin}}(\delta,  \bm{\Sigma}_{\Dtr}^{-1}\bm{\Delta y_{\text{cal}}}),
\end{align}
where $q_{\text{lin}}(\delta,  \bm{\Sigma}_{\Dtr}^{-1}\bm{\Delta y_{\text{cal}}})$ is a monotonically increasing piecewise linear function\footnote{ For $\bm{a} \in \mathbb{R}^{N_{\text{cal}}}$, a permutation $i_1,\ldots, i_{N_{\text{cal}}} \in [1,..., N_{\text{cal}}]$ where $a_{i_j} \leq a_{i_{j+1}}$ for all $j$, and $\delta \in \left[{l}/(N_{\text{cal}}+1),  (l+1)/(N_{\text{cal}}+1) \right]$,
\[
q_{\text{lin}}(\delta, \bm{a}) = 
a_{i_l} + \left(\delta (N+1) - l\right)\left(a_{i_{l+1}}-a_{i_l}\right).  \]} that maps $\delta = j/(N_{\text{cal}}+1)$ to the $j$-th smallest entry of $\bm{\Sigma}_{\Dtr}^{-1}\bm{\Delta y_{\text{cal}}}$, and the entries of the vector
\[\bm{\Sigma}_{\Dtr}^{-1}\bm{\Delta y_{\text{cal}}} = \left( \frac{\Delta y_{\text{cal}}^1}{\sigma_{\Dtr}\left(\bm{\theta}_{\delta}, \bm{x}_{\text{cal}}^1\right)},..., \frac{\Delta y_{\text{cal}}^{N_{\text{cal}}}}{\sigma_{\Dtr}(\bm{\theta}_{\delta}, \bm{x}_{\text{cal}}^{N_{\text{cal}}})}\right)^{\top}\]
correspond to the z-scores of the data under the calibration standard deviation $\sigma_{\Dtr}\left(\bm{\theta}_{\delta}, \cdot\right)$. The original problem \eqref{eq:desiredequationcal} then becomes
\begin{align}
\label{eq:surrogateminimization}
    \begin{split}
        \min_{\bm{\theta}\in \bm{\varTheta}} \ &  \sum^{N_{\text{cal}}}_{i=1}   \left[q_{\text{lin}}(\delta,  \bm{\Sigma}_{\Dtr}^{-1}\bm{\Delta y_{\text{cal}}}) \sigma_{\Dtr}\left(\bm{\theta}_{\delta}, \bm{x}_{\text{cal}}^i\right) \right]^2 ,
    \end{split}
\end{align}
which is considerably easier to solve due to the lack of constraints, and enables us to use gradient-based approaches, since $q_{\text{lin}}(\delta,  \bm{\Sigma}_{\Dtr}^{-1}\bm{\Delta y_{\text{cal}}})$ is differentiable with respect to $\sigma_{\Dtr}(\bm{\theta}_{\delta}, \bm{x}_{\text{cal}}^i)$.

\begin{remark}
The choice of interpolant \eqref{eq:betaasqlin} is due to its simplicity. However, other forms of monotone interpolation are also possible. For high data sizes, the choice of interpolant becomes of little relevance, since we only perform small interpolation steps. 
\end{remark}

\subsection{Calibrated GP for Arbitrary Confidence Level $\delta$}
\label{subsection:arbitrarydelta}

While \eqref{eq:surrogateminimization} is useful for obtaining a sharply calibrated model for a single confidence level $\delta$, solving \eqref{eq:surrogateminimization} multiple times whenever we want to obtain sharply calibrated models for different confidence levels $\delta$ can be time-consuming. Furthermore, interpolating between any two arbitrary solutions of \eqref{eq:surrogateminimization} won't necessarily yield a result close to the desired calibration.
Fortunately, we can leverage \Cref{assumption:montonicity} to show that interpolating between two solutions of \eqref{eq:surrogateminimization} will yield a result close to the desired calibration, provided that we interpolate between two strictly increasing or decreasing sets of hyperparameters. Formally, this is achieved by solving \eqref{eq:surrogateminimization} $N_{\text{cal}}+1$ times to obtain $\beta_0,\beta_{\delta_1}, \ldots, \beta_{\delta_{N}}\in \mathbb{R}$ and $\bm{\theta}_0,\bm{\theta}_{\delta_1}, \ldots, \bm{\theta}_{\delta_{N}}\in \bm{\varTheta}$, subject to two additional constraints. First, the calibration scaling parameters $\beta_{\delta}$ must be monotonically increasing with $\delta$, i.e.,
\[\beta_{\delta_i} \leq \beta_{\delta_j} \quad \text{if} \ {\delta_i} < {\delta_j},\] 
and the calibration hyperparameters $\bm{\theta}_{\delta}$ must be decreasing with $\delta$ if $\beta_{\delta}$ is negative, and increasing if $\beta_{\delta}$ is positive, i.e., 
\[\bm{\theta
}_{\delta_i} \leq \bm{\theta
}_{\delta_j}, \quad \text{if} \ {\delta_i} < {\delta_j} \ \text{and} \ \ \beta_{\delta_i}\geq 0,\]
\[\bm{\theta
}_{\delta_i} \geq \bm{\theta
}_{\delta_j}, \quad \text{if} \ {\delta_i} < {\delta_j} \ \text{and} \ \ \beta_{\delta_j}\leq 0.\]
In other words, the entries of $\bm{\theta}_{\delta}$ are strictly decreasing with $\beta_{\delta}$ up until the sign of $\beta_{\delta}$ switches, after which they are increasing. The reason why we impose these restrictions is that we can then confidently interpolate between any values of $\beta_{\delta}$ and $\bm{\theta}_{\delta}$, since \Cref{assumption:montonicity} implies that the quantile stipulated by $\beta_{\delta}\sigma_{\Dtr}(\bm{\theta},\cdot)$ is monotonically increasing with $\delta$. We then train simple piecewise linear interpolation models $\hat{\beta}:[0,1] \rightarrow \mathbb{R}$ and $\hat{\bm{\theta}}:[0,1] \rightarrow \bm{\varTheta}$, such that $\hat{\beta}(\delta_i) = \beta_{\delta_i}$ and $\hat{\bm{\theta}}(\delta_i) = \bm{\theta}_{\delta_i}$, with the additional constraint that $\hat{\bm{\theta}}(\delta)$  reaches a minimum whenever $\hat{\beta}(\delta)=0$, which can be potentially achieved by adding an artificial vector of training hyperparameters $\bm{\theta}_{\delta}$ for computing $\hat{\bm{\theta}}(\delta)$. Note, however, that any other form of monotone interpolation is also acceptable for obtaining $\hat{\beta}(\delta)$ and $\hat{\bm{\theta}}(\delta)$. The procedure is summarized in \Cref{alg:calibratedgp}.

\begin{remark}
 While lengthscale constraints of the form $\bm{\theta}_{\delta_{j}}\leq\bm{\theta}_{\delta_{j+1}}$ can be easily enforced when solving \eqref{eq:surrogateminimization} by substituting $\bm{\theta}$ with $\bm{\theta}_{\delta_{j}} + \Delta \bm{\theta}$ and minimizing over the logarithm of $\Delta \bm{\theta}$, the constraint $\beta_{\delta_{j}}\leq\beta_{\delta_{j+1}}$ is not enforced in \eqref{eq:surrogateminimization}. However, in practice we were often able to find local minima of \eqref{eq:surrogateminimization} that satisfy this requirement.
\end{remark}

We can then easily show that our approach achieves an arbitrary calibration level as the amount of data grows, provided that we choose the confidence levels $\delta$ accordingly.

\begin{algorithm}[tb]
   \caption{Training Calibration Hyperparameters for Arbitrary Confidence Level}
\label{alg:calibratedgp}
\begin{algorithmic}
   \STATE {\bfseries Input:} kernel $k(\cdot,\cdot,\cdot)$, predictor $\mu_{\Dtr}(\bm{\theta}^R,\cdot)$, calibration data $\Dcal$, confidence levels $\delta_1,...\delta_N$
   \FOR{$i=1$ {\bfseries to} $M$}
   \STATE Compute $\beta_{\delta_1},\bm{\theta}_{\delta_1}, ...,\beta_{\delta_{N_{\text{cal}}}},\bm{\theta}_{\delta_{N_{\text{cal}}}}$ by solving \eqref{eq:betaasqlin} and \eqref{eq:surrogateminimization} subject to 
   \[\bm{\theta
}_{\delta_i} \leq \bm{\theta
}_{\delta_j}, \quad \text{if} \ {\delta_i} < {\delta_j} \ \text{and} \ \ \beta_{\delta_i}\geq 0,\]
\[\bm{\theta
}_{\delta_i} \geq \bm{\theta
}_{\delta_j}, \quad \text{if} \ {\delta_i} < {\delta_j} \ \text{and} \ \ \beta_{\delta_j}\leq 0.\]
   \ENDFOR
   \STATE Fit a continuous, monotonically increasing interpolation model $\hat{\beta}(\delta)$ and a continuous model $\hat{\bm{\theta}}(\delta)$ using the training data $\left\{\delta_i,  \beta_{\delta_i}, \bm{\theta}_{\delta_i}\right\}_{i=1,...,N}$
   \STATE{\bfseries Output:} $\hat{\beta}(\delta)$, $\hat{\bm{\theta}}(\delta)$ 
\end{algorithmic}
\end{algorithm}

\begin{theorem}
\label{theorem:maintheorem}
Let $y$ be absolutely continuous conditioned on $\bm{x}$, let $\mu_{\Dtr}(\bm{\theta}^R,\cdot)$ be a posterior GP mean, and let $\sigma_{\Dtr}(\cdot,\cdot)$ be a GP posterior variance conditioned on $\Dtr$. Then, for any calibration data set
$\Dcal = \left\{\bm{x}_{\text{cal}}^i, y_{\text{cal}}^i\right\} $, choose 
\[\delta_1 = \frac{1}{N_{\text{cal}}+1}, \quad \delta_2 = \frac{2}{N_{\text{cal}}+1}, \  \ldots, \quad \delta_{N_{\text{cal}}}=\frac{N_{\text{cal}}}{N_{\text{cal}}+1},\]
and let $\hat{\beta}(\cdot)$ and $\hat{\bm{\theta}}(\cdot)$ be interpolation models obtained with \Cref{alg:calibratedgp} and confidence levels $\delta_1,...,\delta_{N_{\text{cal}}}$.
% $\beta_{\delta_1}, \ldots, \beta_{\delta_{N_\text{cal}}}\in \mathbb{R}$ be $N_{\text{cal}}$ scalars and $\bm{\theta}_{\delta_1}, \ldots, \bm{\theta}_{\delta_{N_{\text{cal}}}}\in \bm{\varTheta}$ be $N_{\text{cal}}$ vectors of calibration hyperparameters, such that $\beta_{\delta_1}\leq\ldots\beta_{\delta_{N_{\text{cal}}}}$ and 
% \begin{align*}
% \bm{\theta}_{\delta_{i-1}}\geq \bm{\theta}_{\delta_i} &\quad  \text{if} \ \beta_{\delta_{i}}\leq 0 \\
% \bm{\theta}_{\delta_{i}}\leq \bm{\theta}_{\delta_{i+1}} &\quad  \text{if} \ \beta_{\delta_{i}}\geq 0.
% \end{align*} 
% and, for some permutation of of the calibration data, it holds that
% \begin{align*}
% &y_{\text{cal}}^{i_{j+1}}- \mu_{\Dtr}(\bm{\theta}^R,\bm{x}_{\text{cal}}^{i_{j+1}}) - \beta_{\delta_{j+1}}\sigma_{\Dtr}(\bm{\theta}_{{\delta}_{j+1}}, \bm{x}_{\text{cal}}^{i_{j+1}}) \\ < & y_{\text{cal}}^{i_{j}}- \mu_{\Dtr}(\bm{\theta}^R,\bm{x}_{\text{cal}}^{i_j}) - \beta_{\delta_j}\sigma_{\Dtr}(\bm{\theta}_{{\delta}_j}, \bm{x}_{\text{cal}}^{i_{j}})
% \end{align*}
% for all $j=1,...,N_{\text{cal}}$. 
Then
\begin{align*}
    \mathbb{P}_{\bm{x},y \sim \Pi}\Big( y - \mu_{\Dtr}(\hat{\bm{\theta}}(\delta),\bm{x}) \leq \hat{\beta}(\delta)\sigma_{\Dtr}(\hat{\bm{\theta}}(\delta),\bm{x})\Big) \in \left[\delta - \frac{1}{N_{\text{cal}}+1}, \delta + \frac{1}{N_{\text{cal}}+1}\right].
\end{align*}
\end{theorem}
\begin{proof}
The proof can be found in the supplementary material.
\end{proof}

Note that \Cref{theorem:maintheorem} also implies that a single set of calibration parameters $\beta_{\delta}$ and $\bm{\theta}_{\delta}$ obtained by solving \eqref{eq:surrogateminimization}, since we can substitute $\hat{\beta}(\delta) =\beta_{\delta}$
and $\hat{\bm{\theta}}(\delta) =\bm{\theta}_{\delta}$ into \eqref{eq:calibrationmetric}.

\section{Discussion}
\label{section:discussion}

\paragraph{Computational Complexity.} Much like hyperparameter optimization for standard GPs \citep{Rasmussen2006}, the major driver behind the computational complexity in our approach stems from the need to invert the covariance matrix, an operation that scales cubically with the amount of data. In order to alleviate the computational cost of our approach, we can resort to different tools that improve scalability \citep{liu2020gaussian}. One approach consists of employing only a subset of the training data $\D_{\text{tr}}$ to choose the calibration hyperparameters $\bm{\theta}_{\delta}$, and then the full data set to choose $\beta_\delta$. While this potentially leads to a loss in sharpness compared to when using the full data set, it still guarantees a calibrated model. Our technique is also readily applicable to sparse GPs \citep{snelson2005sparse,titsias2009variational}, as \Cref{assumption:montonicity} typically still holds. This option is also explored in numerical experiments, in \Cref{section:experiments}. Moreover, in many settings a specific level of calibration is often required, as opposed to several different ones, e.g., in stochastic model predictive control, where chance constraints corresponding to a fixed risk have to be satisfied \citep{mesbah2016stochastic}. In such settings, we potentially only have to train a single vector of calibration hyperparameters $\bm{\theta}_{\delta}$, which reduces computational cost. 

% \paragraph{Choosing Different Metrics for Training.} Although we choose to minimize the discrepancy between predicted and desired confidence level in \eqref{eq:surrogateminimization}, we typically have multiple degrees of freedom, i.e., hyperparameters, which can be employed to, e.g., minimize other metrics while guaranteeing that \eqref{eq:surrogateminimization} does not deviate too strongly from its minimum. A possibility in this case would be to minimize sharpness, which can be quantified in terms of the sum of calibration variances \citep{kuleshov2018accurate}, while applying constraints to guarantee that the value of \eqref{eq:surrogateminimization} is kept low. Alternatively, a problem with soft-constraints can be employed to achieve a similar result. In our experiments in \Cref{section:experiments}, when solving \eqref{eq:surrogateminimization}, we initialized the hyperparameters with those corresponding to the log-likelihood maximum $\bm{\theta}^{R}$. This way, the initial guess corresponds to a low-complexity model, i.e., one that is sharp, which in turn potentially leads to a solution of \eqref{eq:surrogateminimization} that is also comparatively sharp.

\paragraph{Initialization and Solution.}
The choice of initial hyperparameters can affect
the optimization results considerably, and choosing a good hyperparameter initialization can be challenging, as is true when choosing the hyperparameters for the predictive mean. While this can be partially addressed by employing random restarts, we can also reuse trained models for similar calibration levels, since it is reasonable to expect that only small changes to the calibration hyperparameters are required to achieve a slight increase or decrease in confidence level. Furthermore, we can also simplify the problem by considering only a scaled version of the regression hyperparameters $\bm{\theta}^R$ to compute $\bm{\theta}_{\delta}$, which would reduce the optimization problem to a line search. %Then, we perform multiple optimizations with a successively increasing slope for the sigmoid function, and use the solution from the prior optimization problem as a new guess. This way, the sigmoid functions slowly approximate the step function and the initial hyperparameters will be close enough to the solution for the slope not to be zero.

\begin{figure*}
\centering
\begin{subfigure}[b]{0.32\textwidth}
\includegraphics[width = 1.1\columnwidth, trim={1.9cm 2cm 0cm 0},clip]{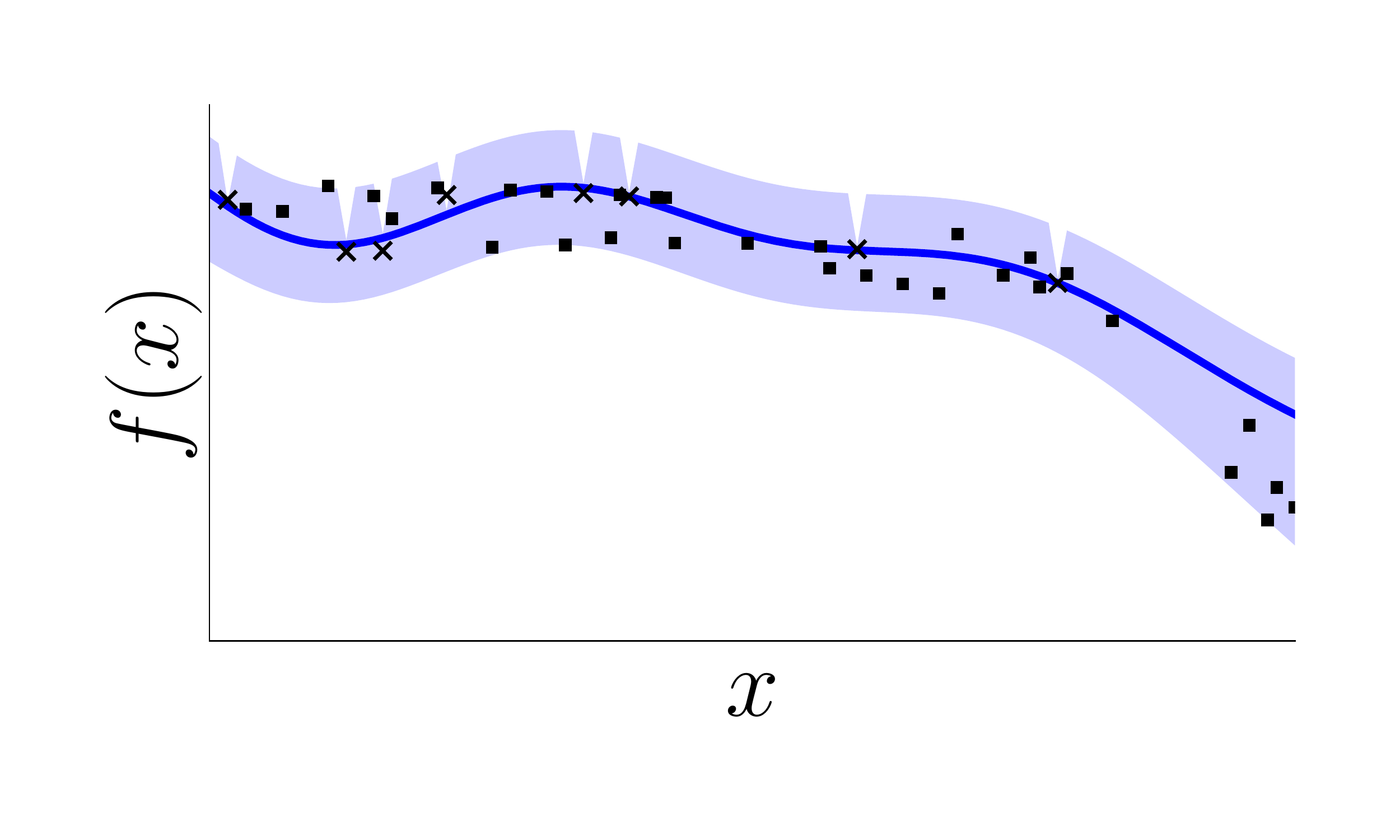}
\caption{Our approach.} 
\label{fig:toy_problem_ours}
\end{subfigure}
\hfill
\begin{subfigure}[b]{0.32\textwidth}
\includegraphics[width = 1.1\columnwidth, trim={1.9cm 2cm 0 0},clip]{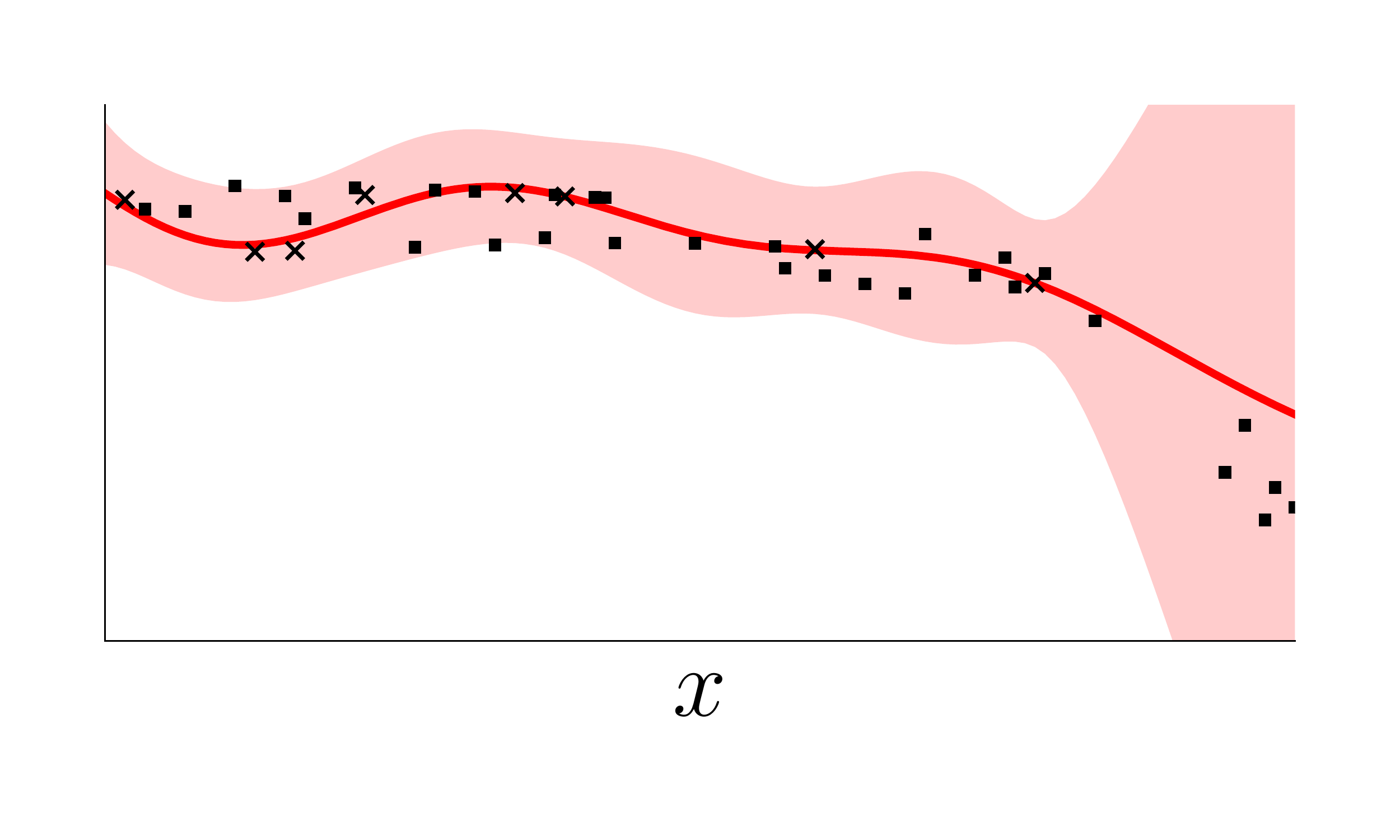}
\caption{\citet{kuleshov2018accurate}.}
\label{fig:toy_prob_kuleshov}
\end{subfigure}
\hfill
\begin{subfigure}[b]{0.32\textwidth}
\includegraphics[width = 1.1\columnwidth, trim={1.9cm 2cm 0cm 0},clip]{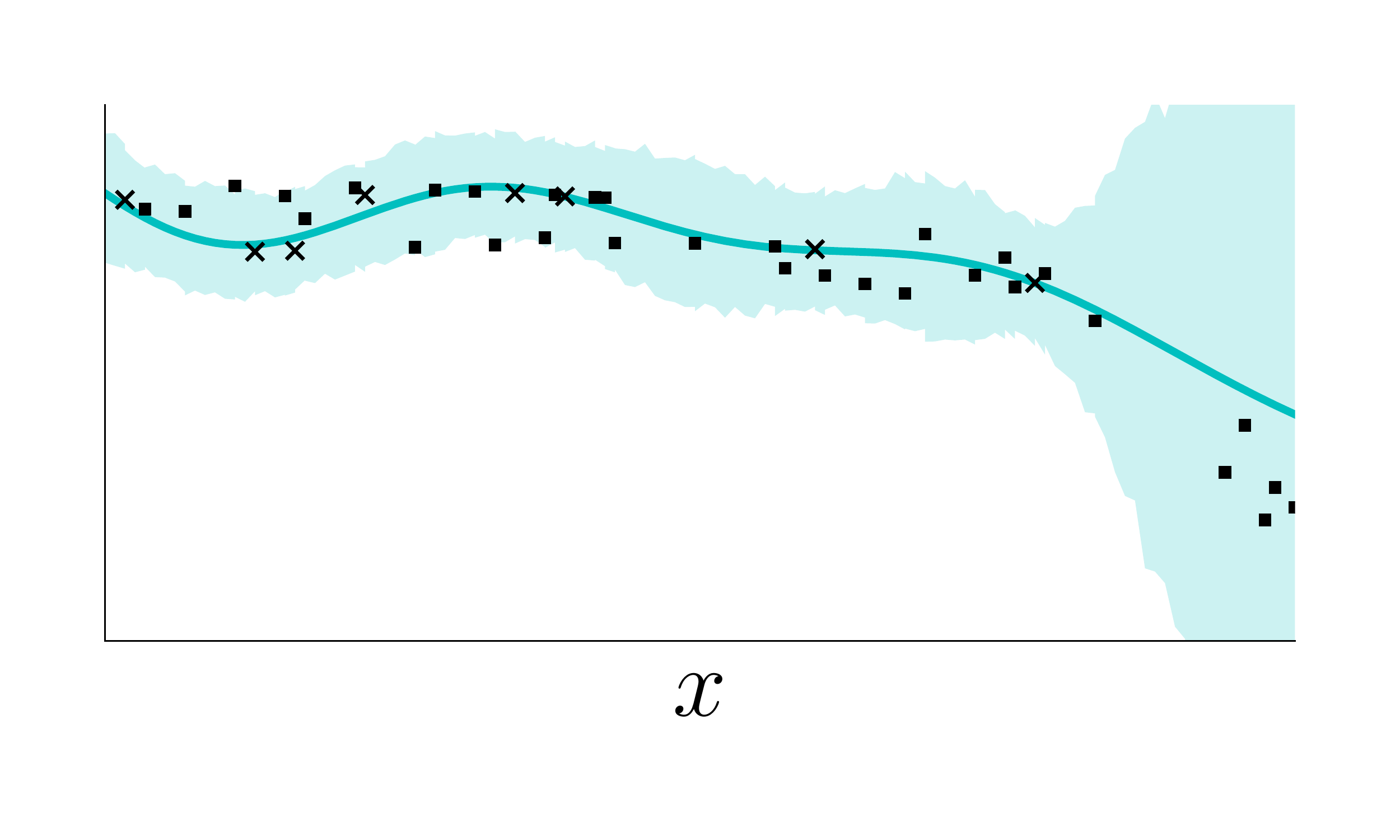}
\caption{\citet{vovk2020conformal}.} 
\label{fig:toy_problem_vovk}
\end{subfigure}
\\
\begin{subfigure}[b]{0.32\textwidth}
\includegraphics[width = 1.1\columnwidth, trim={1.9cm 2cm 0 0},clip]{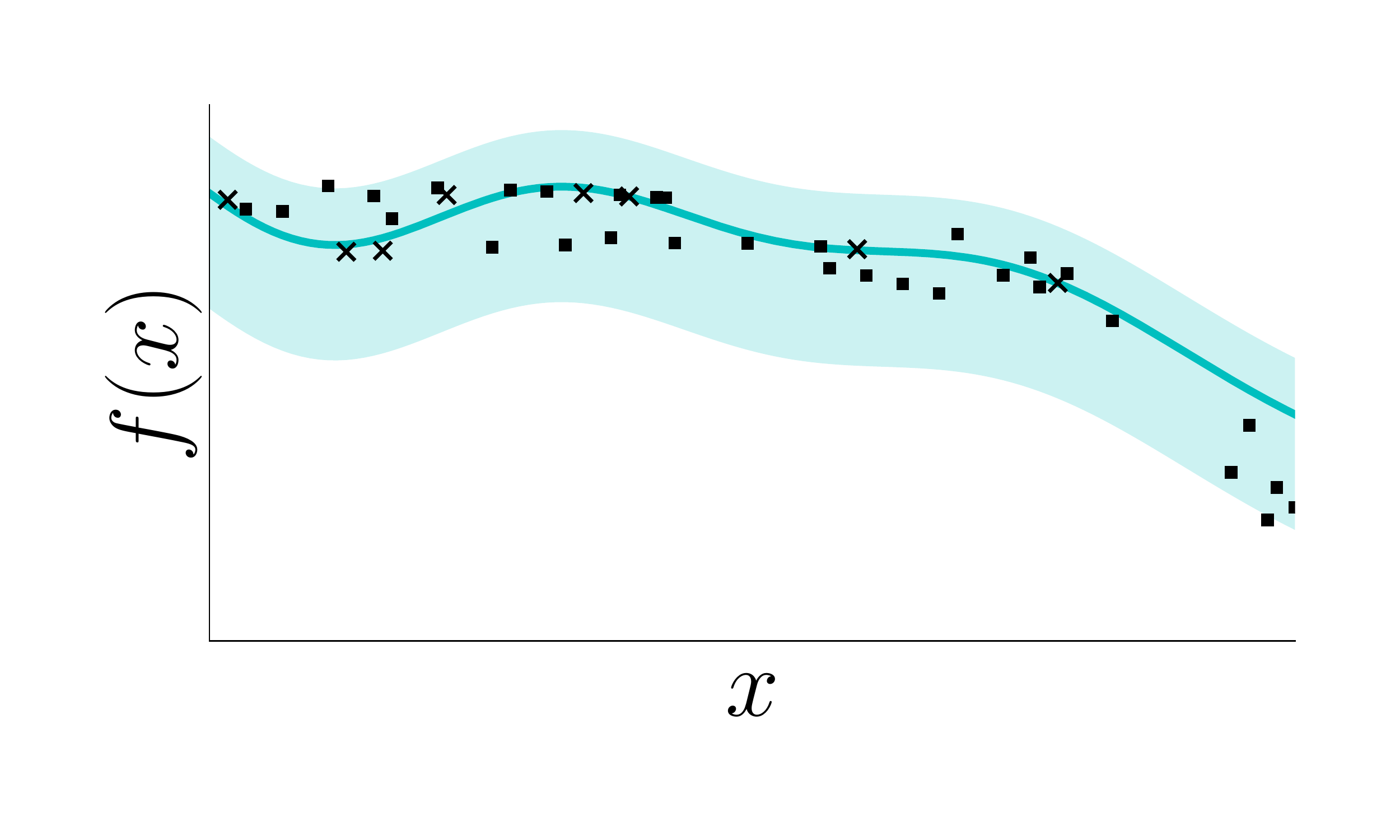}
\caption{Mean $\pm$ constant \citep{marx2022modular}.} 
\label{fig:bostonsharpness}
\end{subfigure}
\hfill
\begin{subfigure}[b]{0.32\textwidth}
\includegraphics[width = 1.1\columnwidth, trim={1.9cm 2cm 0 0},clip]{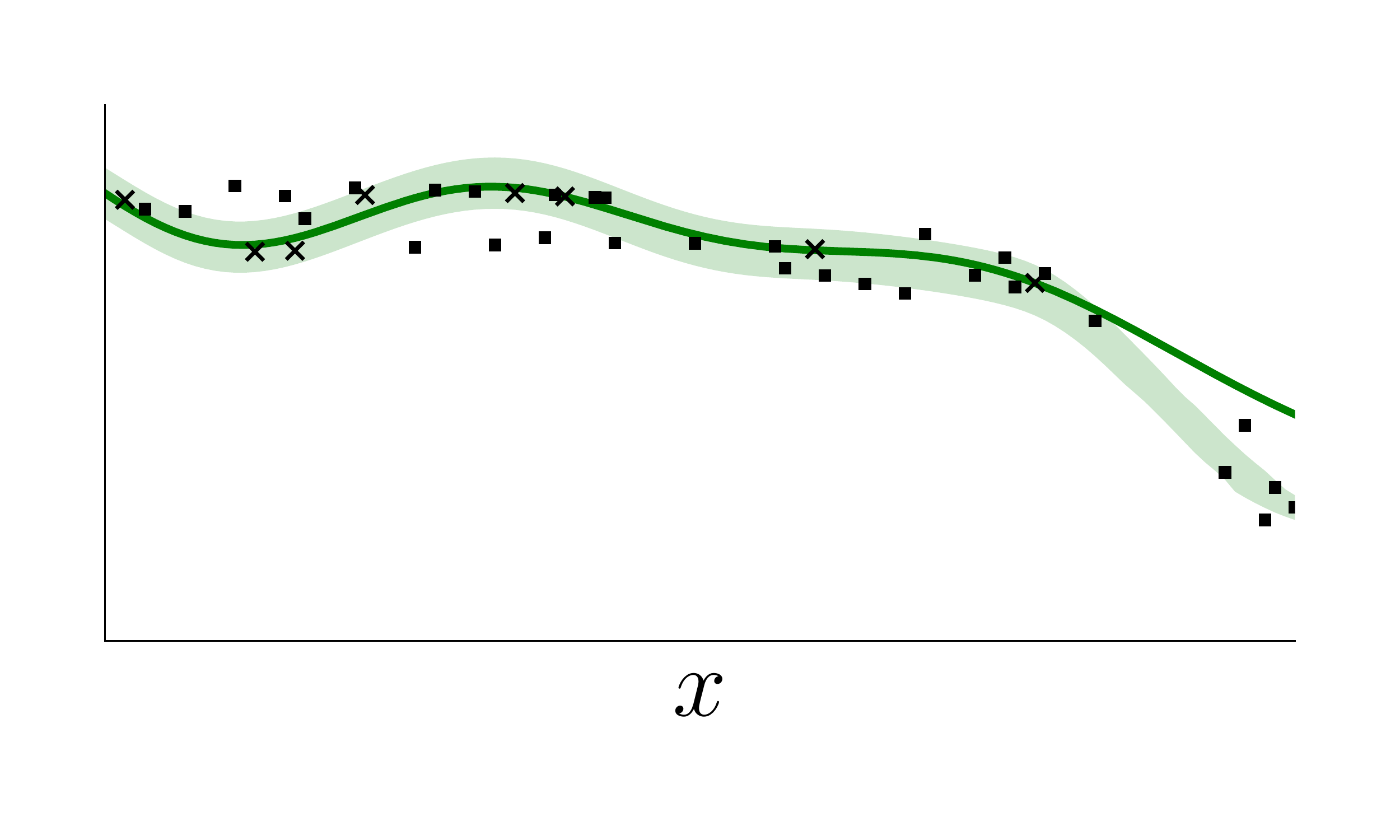}
\caption{Neural network-based recalibrator \citep{kuleshov2022calibrated}.} 
\label{fig:toy_problem_kuleshovNN}
\end{subfigure}
\hfill
\begin{subfigure}[b]{0.32\textwidth}
\includegraphics[width = 1.1\columnwidth, trim={1.9cm 2cm 0 0},clip]{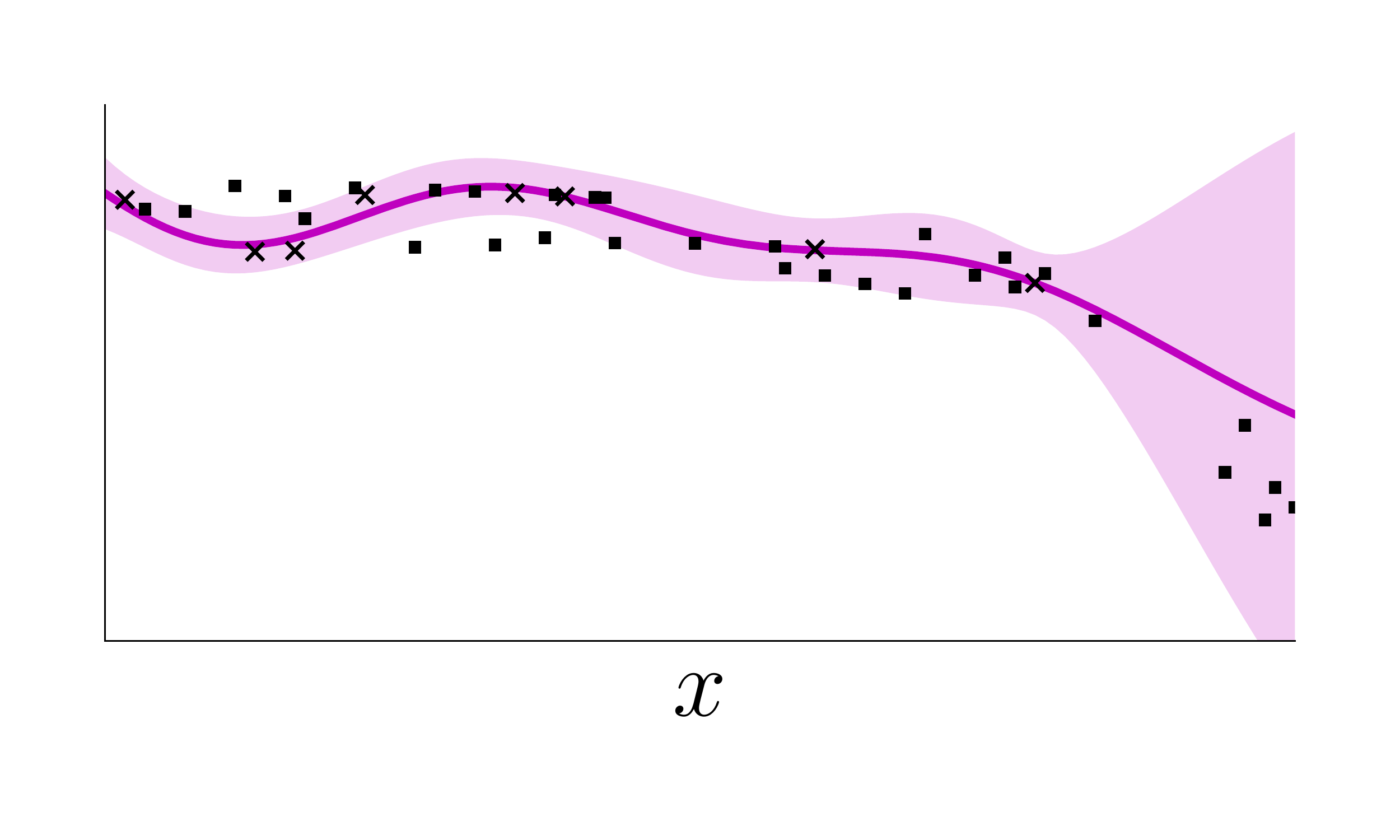}
\caption{Base model (vanilla GP).  \\\hspace{\textwidth}} 
\label{fig:toy_problem_base}
\end{subfigure}
\caption{Centered 99 $\%$ confidence intervals (shaded regions) obtained with our method, vanilla GPs recalibrated using the approaches of \citet{kuleshov2018accurate} and \citet{vovk2020conformal}, a naive vanilla GP, the point-predictor (variance-free) approach proposed in \citet{marx2022modular}, and a naive fully Bayesian GP. Solid lines represent the predictive mean, crosses represent data used to train the base model, and squares represent calibration (holdout) data. Our approach yields a model that is both well-calibrated and sharp. This is because of the added flexibility that comes from being able to also change the lengthscale to design a calibrated model.}
\label{fig:toyproblem}
\end{figure*}

\section{Experiments}
\label{section:experiments}
In this section, we apply and analyze our approach using a toy data set and different regression benchmark data sets from the UCI repository. In the supplementary material, we also compare our approach to that of \cite{capone2022gaussian} when used to obtain uniform error bounds and apply our method to two different Bayesian optimization problems.

The goal is for our approach to obtain a sharp calibrated regression model for each data set in the calibrated regression experiments. We test our approach on various data sets and compare it to the state-of-the-art recalibration approaches by \citet{kuleshov2018accurate} and \citet{vovk2020conformal}, the point predictor (posterior variance-free) approach proposed in \citet{marx2022modular}, as well as the check-score-based approach of \cite{kuleshov2022calibrated}. % Bayesian approach for the vanilla and fully Bayesian GPs. 
The technique proposed by \citet{kuleshov2018accurate} essentially multiplies the vanilla posterior standard deviation $\sigma_{\Dtr}(\bm{\theta}^R,\cdot)$ with the recalibrated z-score, such that the confidence level observed on the calibration data matches that of the desired confidence level. \citet{vovk2020conformal} employ a similar approach, except that random interpolation is employed to compute new scaling values. The point predictor-based method proposed in \citet{marx2022modular} discards the posterior standard deviation $\sigma_{\Dtr}(\bm{\theta}^R,\cdot)$ and computes a constant scalar that is added to the predictive mean and used to compute quantiles everywhere within the input space. The method of \cite{kuleshov2022calibrated} trains a neural network using a quantile loss, which takes base quantiles as inputs and returns new, recalibrated quantiles. As a recalibrator for the method of \cite{kuleshov2022calibrated}, we employ the same neural network architecture suggested in their paper, trained over $200$ epochs, with additional pretraining over $2000$ epochs using a single dataset for the UCI experiments. %When employing the naive vanilla and fully Bayesian GPs, the credible sets are assumed to correspond exactly to confidence intervals. 
In all experiments except kin8nm and Facebook comment volume 2, we employ standard GPs with automatic relevance determination squared-exponential (ARD-SE) kernels and zero prior mean as base models, trained using log-likelihood maximization. For the kin8nm and Facebook comment volume 2 datasets, we employ sparse GPs \citep{titsias2009variational} with zero prior mean, ARD-SE kernels, and $300$ inducing points.

\subsection{Toy Data Set}
\label{subsubsection:toydataset}

The first regression data set corresponds to a one-dimensional synthetic data set, where the results can be easily displayed visually. The main purpose of this section is to give an intuition as to how our approach computes confidence intervals compared to other techniques. We investigate the performance of our approach and compare it to other methods when employed to compute centered $99\%$ confidence intervals. We observe that the confidence intervals obtained with our approach peak less strongly far away from the data while being tight near the data compared to all other approaches except the one of \cite{kuleshov2022calibrated}. This is because we allow the lengthscale to change to obtain a calibrated model. In contrast, all other methods except that of \cite{kuleshov2022calibrated} scale the standard GP posterior variance without changing hyperparameters. The method of \cite{kuleshov2022calibrated}, which uses a neural network as a recalibrator, offers additional flexibility, resulting in sharper confidence intervals. However, calibration is not explicitly enforced during training, resulting in poor calibration. The results are depicted in \Cref{fig:toyproblem}.

\subsection{Benchmark Data Sets}
\label{subsection:ucidatasets}

We now experiment with seven different regression data sets from the UCI repository, two containing over eight thousand data points and requiring sparse GP approximations. The training/calibration/test split is $0.6$, $0.2$, and $0.2$ for all data sets except the Facebook comment volume 2 data set, which contains over $80000$ data points, and where the split is $0.08$, $0.02$, and $0.9$. For the approach of \cite{kuleshov2022calibrated}, we follow the steps in their paper and limit the calibration data size to $500$.

We assess performance by employing diagnostic tools commonly used to assess calibration and sharpness \citep{kuleshov2018accurate,marx2022modular,gneiting2007probabilistic}). The score used to quantify calibration is the calibration error \citep{kuleshov2018accurate}, given by
\begin{align}
\label{eq:calibrationmetric}
    \begin{split}
        \text{cal}\left(\mu_{\Dtr}(\bm{\theta}^{R},\cdot), \beta_{\cdot}, \sigma_{\Dtr}(\cdot,\cdot)\right) = \sum_{j=1}^{m} \left(p_j - \hat{p}_j\right)^2,
    \end{split}
\end{align}
where $p_j$ corresponds to the $j$-th desired confidence level, chosen, e.g., evenly spaced between $0$ and $1$, and $\hat{p}_j$ is the observed confidence level, i.e.,
\begin{align}
\label{eq:observedconfidence}
    \begin{split}
        \hat{p}_j = \frac{\left\vert \left\{ y_t^* \ \big\vert \  \Delta y_t^* \leq  \beta_{p_j}\sigma_{\Dtr}(\bm{\theta}_{p_j}\bm{x}_t^*), t=1,..., T \right\} \right\vert }{T}.
    \end{split}
\end{align}
Here the superscript $*$ denotes test inputs and outputs, $T$ denotes the total number of test points, and $\Delta y_t^* \coloneqq \mu_{\Dtr}(\bm{\theta}^{R},\bm{x}_t^*) - y_t^*$. We employ $m=21$ evenly spaced values between $0$ and $1$ for ${p}_j$.
To measure sharpness, we employ the average length of the $95\%$ confidence interval, the average standard deviation of the predictive distribution, and the average negative log-likelihood of the predictions \citep{gneiting2007probabilistic,marx2022modular}. Note that since every model outputs a quantile for any desired calibration level, the corresponding negative log-likelihood and average standard deviation are well specified. These are computed by employing the cumulative distribution function, obtained by inverting the quantile function specified by each model.

\begin{table*}[t]
\caption{Expected calibration error and sharpness of different methods over $100$ repetitions per experiment. We report the expected calibration error (ECE), the average predictive standard deviation (STD), negative log-likelihood (NLL) and 95$\%$ confidence interval width (95$\%$ CI) obtained with our approach, vanilla GPs recalibrated using the methods of \citet{kuleshov2018accurate} (RK) and \citet{vovk2020conformal} (RV), the variance-free approach proposed in \citet{marx2022modular} (RM), and the neural network-based recalibrator of \cite{kuleshov2022calibrated} (NN). We additionally report performance for the base model (B), which corresponds to a vanilla GP without the holdout data. Lower is better for all metrics. In all experiments except the Facebook2 dataset, our method is sharpest compared to all other methods except that of \cite{kuleshov2022calibrated}. However, \cite{kuleshov2022calibrated} performs more poorly in terms of expected calibration error.}
\label{table:resultsuci}
\vskip 0.15in
\begin{center}
\begin{small}
\begin{sc}
\begin{tabular}{lcccccccr}
\toprule
        Data set & Metric & {Ours} & RK & RV & RM & NN & B \\ 
        \midrule
        & ECE & \bostonpricesece \\
        & STD & \bostonpricesstd \\
        Boston & NLL & \bostonpricesnll \\
        & 95$\%$ CI & \bostonpricesinterval \\
        \midrule
        & ECE & \yachtece \\
        & STD & \yachtstd \\
        yacht & NLL & \yachtnll \\
        & 95$\%$ CI & \yachtinterval \\
        \midrule
        & ECE & \autompgece \\
        & STD & \autompgstd \\
        mpg & NLL & \autompgnll \\
        & 95$\%$ CI & \autompginterval \\
        \midrule
        & ECE & \wineece \\
        & STD & \winestd \\
        wine & NLL & \winenll \\
        & 95$\%$ CI & \wineinterval \\
        \midrule
        & ECE & \cementece \\
        & STD & \cementstd \\
        concrete & NLL & \cementnll \\
        & 95$\%$ CI & \cementinterval \\
        \midrule
        & ECE & \kineightnmece \\
        & STD & \kineightnmstd \\
        kin8nm & NLL & \kineightnmnll \\
        & 95$\%$ CI & \kineightnminterval \\
        \midrule
        & ECE
        & \facebookece \\
        & STD & \facebookstd \\
        Facebook2 & NLL & \facebooknll \\
        & 95$\%$ CI & \facebookinterval \\
% Vanilla     & 0.73± 0.036  &  0.85± 0.064   &0.5± 0.043 & 1.5± 0.033 & 1.2± 0.18 &-0.51± 0.021 & 0.51\\
% Full B.     & - &  -1.2± 0.098
%  &0.64± 0.16& - & - &0.12 \\
\bottomrule
\end{tabular}
\end{sc}
\end{small}
\end{center}
\vskip -0.1in
\end{table*}

% \subsection{Toy Example}
% \label{subsection:syntheticdata}
We carried out each experiment $100$ times and report the resulting average expected calibration error, standard deviation, negative log-likelihood, and length of the centered 95$\%$ confidence intervals in \Cref{table:resultsuci}. Our approach performs best or marginally worse than all other calibration approaches regarding expected calibration error. This is to be expected from \Cref{theorem:maintheorem}. Furthermore, it outperforms all approaches except that of \cite{kuleshov2022calibrated} in sharpness. However, the improved sharpness of the method of \cite{kuleshov2022calibrated} comes at the expense of calibration.

% \begin{figure*}
% \centering
% \includegraphics[scale=0.3]{legend.pdf}
% \begin{subfigure}[b]{0.32\textwidth}
% \includegraphics[width = 1.1\columnwidth]{calibrationbostonprices404.pdf}
% \caption{Confidence intervals corresponding to an expected confidence level of $1\%$ obtained with our approach, that of \citet{kuleshov2018accurate}, a vanilla GP, and a fully Bayesian GP. The scaling factor for the two latter approaches is computed using the inverse error function.} 
% \label{fig:bostoncalibration}
% \end{subfigure}
% \hfill
% \begin{subfigure}[b]{0.32\textwidth}
% \includegraphics[width = 1.1\columnwidth]{sharpnessbostonprices404.pdf}
% \caption{Logarithm of sharpness over expected rate of constraint satisfaction. Sharpness corresponds to the sum over the scaled GP posterior variances used to estimate the confidence intervals for each desired confidence level.
% ~\\
% } 
% \label{fig:bostonsharpness}
% \end{subfigure}
% \caption{Detailed results for Boston house price data set. The calibration and sharpness for our method are shown in blue, that of a naive vanilla GP in green, that of a naive fully Bayesian GP in magenta, and that of a vanilla GP rescaled using the approach by \citet{kuleshov2018accurate} in red. Our approach yields the most accurately calibrated model for all confidence levels. Sharpness is only second to the naive approaches at low expected confidence levels. Note, however, that the latter perform very poorly with respect to calibration.}
% \label{fig:bostonhouseprices}
% \end{figure*}

\section{Conclusion}
\label{section:conclusion}

We have presented a calibration method for Gaussian process regression that leverages the monotonicity properties of the kernel hyperparameters to obtain sharp calibrated models. We show that, under reasonable assumptions, our method yields an accurately calibrated model as the size of data used for calibration increases. When applied to different regression benchmark data sets, our approach was shown to be competitive in sharpness compared to state-of-the-art recalibration methods without sacrificing calibration performance. It is worth stressing that, though the tools presented here emerge naturally from a Gaussian process setting, we do not require our predictor to be a Gaussian process to obtain theoretical guarantees. In future work, we aim to leverage similar monotonicity characteristics to get sharply calibrated models using tools different from Gaussian processes. Furthermore, we aim to experiment with inducing variables as hyperparameters when optimizing the models for sharpness.

\section*{Acknowledgements}

This work was supported in part by the European Research
Council Consolidator Grant Safe data-driven control for
human-centric systems (CO-MAN) under grant agreement
number 864686.

\bibliographystyle{icml2022-custom}
\bibliography{AllPhDReferences}

\newpage
{\centering \Large \textbf{Sharp Calibrated Gaussian Processes - Supplementary Material}}

\section*{Proof of \Cref{theorem:maintheorem}}
\label{appendix:proofoftheorem}
For completeness, we state Theorem 1 from \citet{marx2022modular} here in adapted form, which we then use to prove \Cref{theorem:maintheorem}.
\begin{lemma}[\citet{marx2022modular}, Theorem 1]
\label{lemma:marx}
Let $\varphi: \X \times \mathbb{R} \rightarrow \mathbb{R}$ be a function such that that $\varphi(\bm{x}, y)$, $\bm{x}, y\sim \Pi$ is an absolutely continuous random variable and, for any fixed $\bm{x}^* \in \X$, $\varphi(\bm{x}^*, \cdot)$ is strictly monotonically increasing. Furthermore, for a set of calibration data $\Dcal = \left\{\bm{x}_{\text{cal}}^i, y_{\text{cal}}^i\right\} $ with $N_{\text{cal}} = \vert\Dcal \vert$ and a permutation $i_1, \ldots, i_{N_{\text{cal}}}\in [1,2,\ldots,N_{\text{cal}}]$ such that
\[
\varphi(\bm{x}_{\text{cal}}^{i_j}, y_{\text{cal}}^{i_j}) < \varphi(\bm{x}_{\text{cal}}^{i_{j+1}}, y_{\text{cal}}^{i_{j+1}}),
\] 
let $H: \mathbb{R} \rightarrow [0,1]$ be a monotonically non-decreasing function, such that $H(\varphi(\bm{x}_{\text{cal}}^{i_j}, y_{\text{cal}}^{i_j})) = \frac{j}{N_{\text{cal}}+1}$ holds for all $j=1,\ldots, N_{\text{cal}}$. Then
\[
\mathbb{P}_{\bm{x},y \sim \Pi}\Big(H(\varphi(\bm{x}, y)) \leq \delta  \Big) \in \left[\delta- \frac{1}{N_{\text{cal}}+1}, \delta + \frac{1}{N_{\text{cal}}+1}\right]  \quad \forall \ \delta \in [0, 1].
\]
\end{lemma}

The idea behind the proof of \Cref{theorem:maintheorem} is to show that the solution $\varphi(\bm{x}, y)$ of the implicit equation
\begin{align}
\label{eq:implicitequation}
   y - \mu_{\Dtr}(\bm{\theta}^R,\bm{x}) - \hat{\beta}(\varphi(\bm{x}, y)) \sigma\left(\hat{\bm{\theta}}\left(\varphi(\bm{x},y)\right), \bm{x} \right) = 0
\end{align}
satisfies the requirements stipulated by
\Cref{lemma:marx}, where $\tilde{\beta}(\delta)$ and $\tilde{\bm{\theta}}(\delta)$ are \textit{arbitrary} continuous functions such that 
\begin{align}
\label{eq:conditions1}
\begin{split}
    &\lim_{\delta\rightarrow \infty}\tilde{\beta}(\delta) = \infty, \qquad \lim_{\delta\rightarrow -\infty}\tilde{\beta}(\delta) = -\infty,\\
    & \lim_{\delta\rightarrow \infty}\tilde{\bm{\theta}}(\delta) = \infty, \qquad \lim_{\delta\rightarrow -\infty}\tilde{\bm{\theta}}(\delta) = \infty,
    \end{split}
\end{align}
\begin{align}
\label{eq:conditions2}
\begin{split}
& \tilde{\beta}(\delta) \ \text{is strictly monotonically increasing for all $\delta \in \mathbb{R}$} \\
&\tilde{\bm{\theta}}(\delta) \ \text{is monotonically increasing for all} \ \delta \in \{\delta \in \mathbb{R} \ \vert  \ \tilde{\beta}(\delta) >0 \}\\
    & \tilde{\bm{\theta}}(\delta) \ \text{is monotonically decreasing for all} \ \delta \in \{\delta \in \mathbb{R} \ \vert  \ \tilde{\beta}(\delta) <0 \}.
    \end{split}
\end{align} $\tilde{\bm{\theta}}(\delta)$ is monotonically increasing with respect to $\delta$ for all $\delta \in \{\delta \in \mathbb{R} \ \vert  \ \tilde{\beta}(\delta) >0 \}$, and monotonically decreasing with respect to $\delta$ for all $\delta \in \{\delta \in \mathbb{R} \ \vert  \ \tilde{\beta}(\delta) <0 \}$. Note that the functions $\hat{\beta}(\delta)$ and $\hat{\bm{\theta}}(\delta)$ can be easily extended within the real axis to satisfy the requirements mentioned above, which means that they are contained within the set from which $\tilde{\beta}(\delta)$ and $\tilde{\bm{\theta}}(\delta)$. The reason why we choose arbitrary $\tilde{\beta}(\delta)$ and $\tilde{\bm{\theta}}(\delta)$, as opposed to the functions $\hat{\beta}(\delta)$ and $\hat{\bm{\theta}}(\delta)$, is because we need $\varphi(\bm{x}, y)$ to be independent of the calibration data $\Dcal$ in order to be able to employ \Cref{lemma:marx}. Showing that $\varphi(\bm{x}, y)$ satisfies the requirements of \Cref{lemma:marx} for any $\tilde{\beta}(\delta)$ and $\tilde{\bm{\theta}}(\delta)$ then implies that we can also choose any function within this class that minimizes sharpness, meaning that these properties also extend to $\hat{\beta}(\delta)$ and $\hat{\bm{\theta}}(\delta)$.

To prove \Cref{theorem:maintheorem}, we will require the following result.

\begin{lemma}
\label{lemma:monotonicequation}
Consider the regressor $\mu_{\Dtr}(\bm{\theta}^R,\cdot)$, and let $\tilde{\beta}(\delta)$ and $\tilde{\bm{\theta}}(\delta)$ be functions that satisfy \eqref{eq:conditions1} and \eqref{eq:conditions2}. Then, for arbitrary fixed $y$ and $\bm{x}$,
\begin{align}
\label{eq:implicitequation}
   y - \mu_{\Dtr}(\bm{\theta}^R,\bm{x}) - \tilde{\beta}(\delta) \sigma\left(\tilde{\bm{\theta}}\left(\delta\right), \bm{x} \right)
\end{align}
is strictly monotonically decreasing with $\delta$.
\end{lemma}
\textit{Proof.} The proof follows directly from \Cref{assumption:montonicity} and the properties \eqref{eq:conditions1} and \eqref{eq:conditions2}. \qed

\textit{Proof of \Cref{theorem:maintheorem}.}
Let $\tilde{\beta}(\delta)$ and $\tilde{\bm{\theta}}(\delta)$ be functions that satisfy \eqref{eq:conditions1} and \eqref{eq:conditions2}. Due to \Cref{lemma:monotonicequation}, we can define the function $\varphi: \X \times \mathbb{R}\rightarrow [0,1]$ as the unique solution to the implicit equation
\begin{align}
\label{eq:implicitequation}
   y - \mu_{\Dtr}(\bm{\theta}^R,\bm{x}) - \tilde{\beta}(\varphi(\bm{x}, y)) \sigma\left(\tilde{\bm{\theta}}\left(\varphi(\bm{x},y)\right), \bm{x} \right) = 0.
\end{align}
Note that, since $y - \mu_{\Dtr}(\bm{\theta}^R,\bm{x})$ is strictly monotonically increasing with $y$, $\varphi(\bm{x},y)$ is a strictly monotonically increasing function of $y$ for any fixed $\bm{x}$. Furthermore, since $y$ is absolutely continuous,
\[y_{\text{cal}}^{i} - \mu_{\Dtr}(\bm{\theta}^R,\bm{x}_{\text{cal}}^{i}) \neq y_{\text{cal}}^{j} - \mu_{\Dtr}(\bm{\theta}^R,\bm{x}_{\text{cal}}^{j})\]
holds for all $i\neq j$ almost surely, which implies $\varphi(\bm{x}_{\text{cal}}^{i}, y_{\text{cal}}^{i}) \neq \varphi(\bm{x}_{\text{cal}}^{j}, y_{\text{cal}}^{j})$ for all $i\neq j$ almost surely. Hence, $\varphi(\bm{x},y)$, $\bm{x}, y\sim \Pi$, corresponds to an absolutely continuous random variable. Hence,  given any monotonically non-decreasing function $H(\cdot)$ that satisfies the requirement \[H\left(\varphi(\bm{x}_{\text{cal}}^{i_j}, y_{\text{cal}}^{i_j})\right) = \frac{j}{N_{\text{cal}}+1},\]
\Cref{lemma:marx} implies that
\begin{align}
\label{eq:resultforarbitraryphi}
\mathbb{P}_{\bm{x},y \sim \Pi}\Big(H\left(\varphi(\bm{x}, y)\right) \leq \delta  \Big) \in  \left[\delta- \frac{1}{N_{\text{cal}}+1}, \delta + \frac{1}{N_{\text{cal}}+1}\right]  \quad \forall \ \delta \in [0, 1].
\end{align}
Since $\tilde{\beta}(\delta)$ and $\hat{\bm{\theta}}\left(\delta\right)$ are arbitrary, and $\hat{\beta}(\delta)$ and $\tilde{\bm{\theta}}\left(\delta\right)$ are continuous and also satisfy \eqref{eq:conditions1} and \eqref{eq:conditions2} within $\delta \in [0,1]$, we can substitute $\varphi(\cdot,\cdot)$ in \eqref{eq:resultforarbitraryphi} with $\hat{\varphi}(\cdot,\cdot)$, which is the unique solution of the implicit equation
\begin{align}
\label{eq:impliciteqforourphi}
   y - \mu_{\Dtr}(\bm{\theta}^R,\bm{x}) - \hat{\beta}(\hat{\varphi}(\bm{x}, y)) \sigma\left(\hat{\bm{\theta}}\left(\hat{\varphi}(\bm{x},y)\right), \bm{x} \right) = 0.
\end{align}
Now, in the particular case of $\hat{\varphi}(\cdot,\cdot)$, due to \eqref{eq:surrogateminimization}, we have that
\[\hat{\varphi}(\bm{x}_{\text{cal}}^{i_j},y_{\text{cal}}^{i_j}) = \frac{j}{N_{\text{cal}}+1 },\]
meaning that $H(\hat{\varphi}(\bm{x}_{\text{cal}}^{i_j},y_{\text{cal}}^{i_j})) = \hat{\varphi}(\bm{x}_{\text{cal}}^{i_j},y_{\text{cal}}^{i_j})$, i.e., \eqref{eq:resultforarbitraryphi} holds for $\varphi(\cdot,\cdot) = \hat{\varphi}(\cdot,\cdot)$ and the identity function $H(\delta) = \delta$. Furthermore, since $\hat{\varphi}(\cdot,\cdot)$ is uniquely defined by the implicit equation \eqref{eq:impliciteqforourphi} and $\hat{\beta}(\delta) \sigma_{\Dtr}(\hat{\bm{\theta}}(\delta), \bm{x} )$ is monotonically increasing with $\delta$, this in turn implies
\begin{align*}
&\mathbb{P}_{\bm{x},y \sim \Pi}\Big(\hat{\varphi}(\bm{x}, y) \leq \delta \Big) = \mathbb{P}_{\bm{x},y \sim \Pi}\Bigg(\tilde{\beta}(\hat{\varphi}(\bm{x}, y)) \sigma\left(\tilde{\bm{\theta}}\left(\hat{\varphi}(\bm{x},y)\right), \bm{x} \right) \leq \tilde{\beta}(\delta) \sigma\left(\tilde{\bm{\theta}}\left(\delta\right), \bm{x} \right) \Bigg) \\
=& \mathbb{P}_{\bm{x},y \sim \Pi}\Bigg(y - \mu_{\Dtr}(\bm{\theta}^R,\bm{x}) \leq \tilde{\beta}(\delta) \sigma\left(\tilde{\bm{\theta}}\left(\delta\right), \bm{x} \right) \Bigg).
\end{align*}
Since $\tilde{\beta}(\delta)$ and $\hat{\bm{\theta}}\left(\delta\right)$ are arbitrary, and $\hat{\beta}(\delta)$ and $\tilde{\bm{\theta}}\left(\delta\right)$ which, together with \eqref{eq:resultforarbitraryphi}, implies the desired result. \qed

\newpage

\section*{Comparison with \cite{capone2022gaussian}}

In this section, we briefly examine how our approach compares to that of \cite{capone2022gaussian} when used to compute uniform error bounds, i.e., $100$ percent credible intervals, for three different data sets. We carried out each experiment $10$ times. In the following, we report the rate of uniform error bound violation and the average length of the $100$ percent credible intervals. The method of \cite{capone2022gaussian} is purely Bayesian and thus heavily dependent on the prior. The resulting credible intervals are well-calibrated, i.e., they cover most of the data. However, our approach is much better regarding sharpness. This is because \cite{capone2022gaussian} is Bayesian and requires symmetric intervals, whereas our approach is frequentist and allows for asymmetric credible intervals. Our approach also exhibits a lower rate of uniform error bound violations than \cite{capone2022gaussian} in most cases, which suggests that a frequentist approach is more adequate for computing uniform error bounds than a Bayesian one.

\begin{table*}[th]
\caption{Rate of uniform error bound violation (RUEBV) and 100$\%$ confidence interval width obtained with our approach and that of \cite{capone2022gaussian}. Lower is better for all metrics.}
\label{table:standarddevuci}
\vskip 0.15in
\begin{center}
\begin{small}
\begin{sc}
\begin{tabular}{lcccccccr}
\toprule
        Data set & Metric & {Ours} & \cite{capone2022gaussian} \\ 
        \midrule
        & Rate of uniform error bound violation & {0.00376} & \textbf{0.000172} \\
        
        Boston & Length of 100$\%$ CI & \textbf{1.2} & 28.3 \\
        \midrule
        
        % yacht & NLL & \textbf{0.6 ± 0.21} & 0.8± 0.086  \\
        % & 95$\%$ CI & \textbf{1.8 ± 0.26 } & 2.3± 0.33  \\
        % \midrule
        & Rate of uniform error bound violation & \textbf{0.004}  & 0.0065 \\
        mpg 
        & Length of 100$\%$ CI & \textbf{1.7}  &  24.13  \\
        \midrule

        & Rate of uniform error bound violation &  \textbf{0.00072} & {0.00096} \\
        wine & Length of 100$\%$ CI & \textbf{4.7} & 24.8  \\
        % \midrule
        % concrete & NLL & \textbf{0.91 ± 0.043} & 1± 0.046  \\
        % & 95$\%$ CI & \textbf{2.5 ± 0.04} & 2.9± 0.044 \\
        
        % \midrule
        % & ECE & \textbf{0.00018 ± 0.000067} & 0.0022± 0.001 \\
        % & STD & 0.14 ± 0.0029 & \textbf{0.1± 0.0015} \\
        % kin8nm & NLL & -0.56 ± 0.013 & \textbf{-1± 0.012} \\
        % & 95$\%$ CI & 0.53 ± 0.019 & \textbf{0.34± 0.0078}  \\

% Vanilla     & 0.73± 0.036  &  0.85± 0.064   &0.5± 0.043 & 1.5± 0.033 & 1.2± 0.18 &-0.51± 0.021 & 0.51\\
% Full B.     & - &  -1.2± 0.098
%  &0.64± 0.16& - & - &0.12 \\
\bottomrule
\end{tabular}
\end{sc}
\end{small}
\end{center}
\vskip -0.1in
\end{table*}

\newpage

\newpage

\subsection*{Bayesian Optimization}
\label{subsection:bayesianoptimization}
We now investigate how the proposed calibration approach can be employed in a Bayesian optimization context using two commonly used benchmark functions, the Ackley and Rosenbrock functions.

In Bayesian optimization, the goal is to find a point in input space that maximizes an unknown function $f(\cdot)$. In particular, we investigate how our calibrated GP bound performs when used as an upper confidence bound (UCB) for a GP-UCB type acquisition function. Simply put, given a data set ${\mathcal{D}_t}$ of size $t$, the GP-UCB algorithm chooses a query point by maximizing the acquisition function 
\begin{align}
    \bm{x}^*_{t+1} = \arg\max_{\bm{x}} \mu_{\mathcal{D}_t}(\bm{\theta}^{R},\bm{x}) + \beta_{\mathcal{D}_t}\sigma_{\mathcal{D}_t}(\bm{\theta}^{R},\bm{x}),
\end{align}
where $\beta_{\mathcal{D}_t}$ is a tuning parameter that stipulates the trade-off between exploration and exploitation, and may or may not depend on the data set $\mathcal{D}_t$. It has been shown that if the unknown function $f(\cdot)$ belongs to the RKHS associated with the kernel $k(\bm{\theta}^{R},\cdot,\cdot)$, and $\beta_{\mathcal{D}_t}$ is chosen sufficiently large, then the GP-UCB achieves sublinear regret \citep{chowdhury2017kernelized}. However, both assumptions typically cannot be verified in practice, and choosing both the kernel $k(\bm{\theta}^{R},\cdot,\cdot)$ and the scaling factor $\beta_{\mathcal{D}_t}$ in a principled manner remains an open problem. We propose employing the modified acquisition function 
\begin{align}
    \bm{x}^* = \arg\max_{\bm{x}} \mu_{\mathcal{D}_t}(\bm{\theta}^{R},\bm{x}) + \beta_{\delta}\sigma_{\mathcal{D}_t}(\bm{\theta}_{\delta},\bm{x}),
\end{align}
where the hyperparameters $\bm{\theta}_{\delta}$ are obtained via a calibrated model and a suitable choice of confidence parameter $\delta$. In the experiments, we set $\beta_{\mathcal{D}_t}=1$ and compute the calibrated hyperparameters by setting $\delta=0.01$, meaning that we set expect only one percent of the evaluations to lie outside the confidence region. Note that even though the underlying function is fixed, it is reasonable to expect that some of the data lies outside the confidence region due to noise, and we can only expect the data to lie fully within the confidence region in the noiseless case, which we do not consider in this paper. Furthermore, we refrain from retraining the hyperparameters after each data point is collected, following the convention of other Bayesian optimization approaches \citep{Srinivas2012,chowdhury2017kernelized}. While this does not enable us to employ the theoretical guarantees developed in \Cref{section:proposedapproach}, it reduces computational time significantly. We additionally compare our results to the vanilla UCB algorithm, where the hyperparameters, chosen via log-likelihood maximization, are identical for both the posterior mean and variance, and we set $\beta_{\mathcal{D}_t}=2$.

We evaluate the results both in terms of cumulative regret and simple regret. Cumulative regret after $T$ steps corresponds to the metric
\begin{align}
    R_T^{\text{cumul}} = \sum_{t=1}^T  \left(\max_{\bm{x}\in\mathcal{X}} f(\bm{x}) - f(\bm{x}_t) \right),
\end{align}
whereas simple regret is given by
\begin{align}
    R_T^{\text{simple}} = \max_{\bm{x}\in\mathcal{X}} f(\bm{x}) - \max_{t\leq T} f(\bm{x}_t) .
\end{align}
Typically, a Bayesian optimization algorithm is deemed useful if cumulative regret exhibits sublinear growth, implying that the average regret goes to zero. Simple regret, by contrast, corresponds to the best query among all past queries and is an important metric whenever evaluation costs are low \citep{berkenkamp2019no}.

In the case of the Ackley experiment, our approach typically chose lengthscales that were smaller than those computed via likelihood maximization. This results in more exhaustive exploration than vanilla UCB, which in turn means that local minima are explored more carefully before the focus of the optimization is shifted elsewhere. This results in better performance than when using vanilla UCB, both in terms of cumulative and simple regret. The results correspond to the top two figures in \Cref{fig:regret}.

In contrast to the Ackley experiment, in the Rosenbrock experiment our approach selects lengthscales that are larger than those suggested by the likelihood maximum. Roughly speaking, this means that the confidence intervals produced by the likelihood maximum hyperparameters are too conservative, and our approach attempts to compensate for this by indicating more confidence in the posterior mean obtained with the vanilla GP. This means that local minima are explored less meticulously than with the vanilla UCB algorithm. This choice is justified by the cumulative regret obtained with our approach, as it is slightly smaller than that obtained by the vanilla UCB algorithm. However, this also results in worse simple regret than the vanilla UCB algorithm, which is intuitive, as our approach opts to explore local minima less accurately than the vanilla UCB algorithm. We also note that both algorithms converge towards the same simple regret as the number of iterations increases. The results correspond to the bottom two figures in \Cref{fig:regret}.

\begin{figure*}[t]
\centering
\begin{subfigure}[b]{0.495\textwidth}
\includegraphics[width = 0.995\columnwidth]{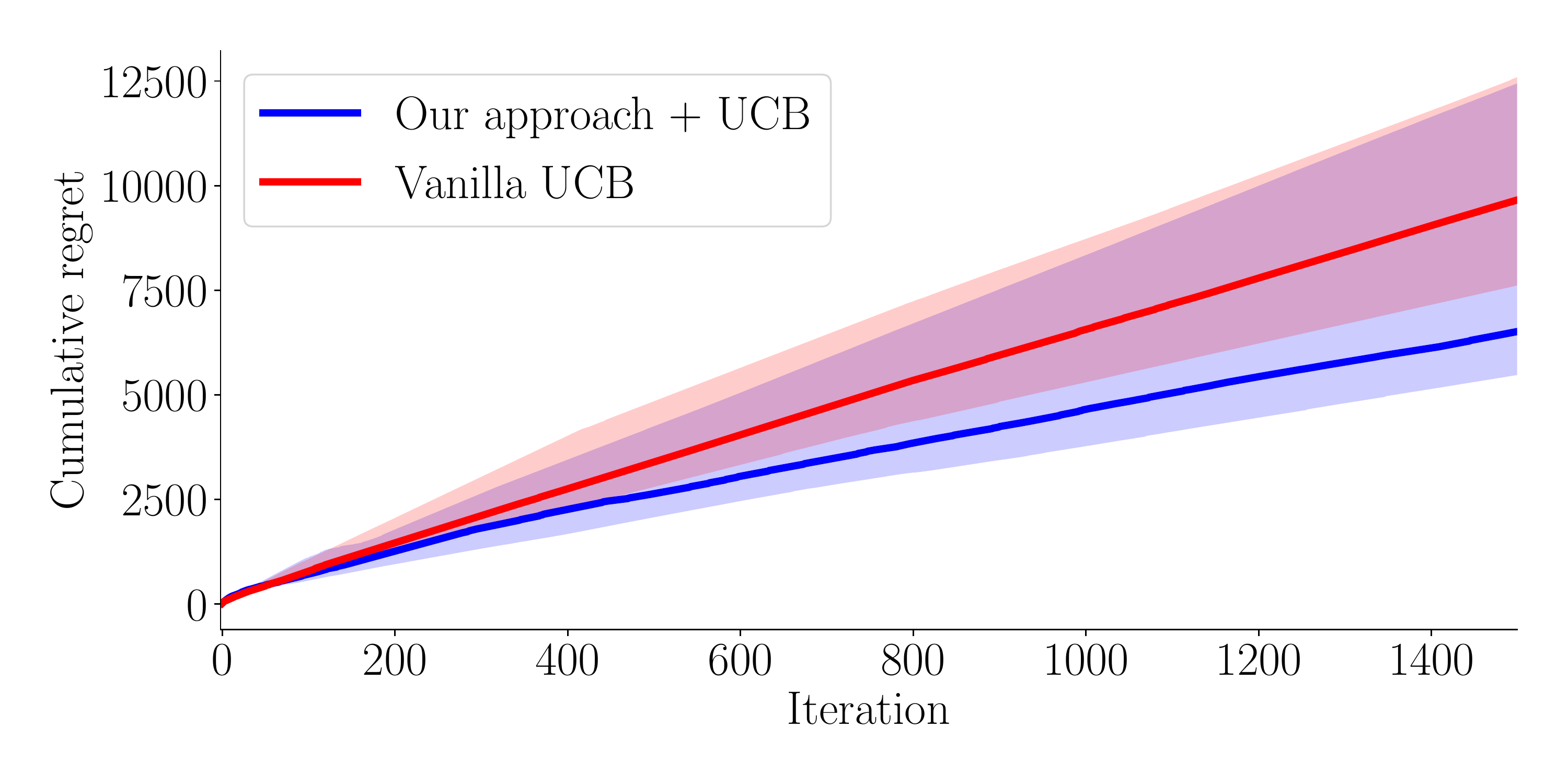}
\caption{Cumulative regret $R_T^{\text{cumul}}$ of Ackley experiment.} 
\label{fig:ackleycumulative}
\end{subfigure}
\hfill
\begin{subfigure}[b]{0.495\textwidth}
\includegraphics[width = 0.995\columnwidth]{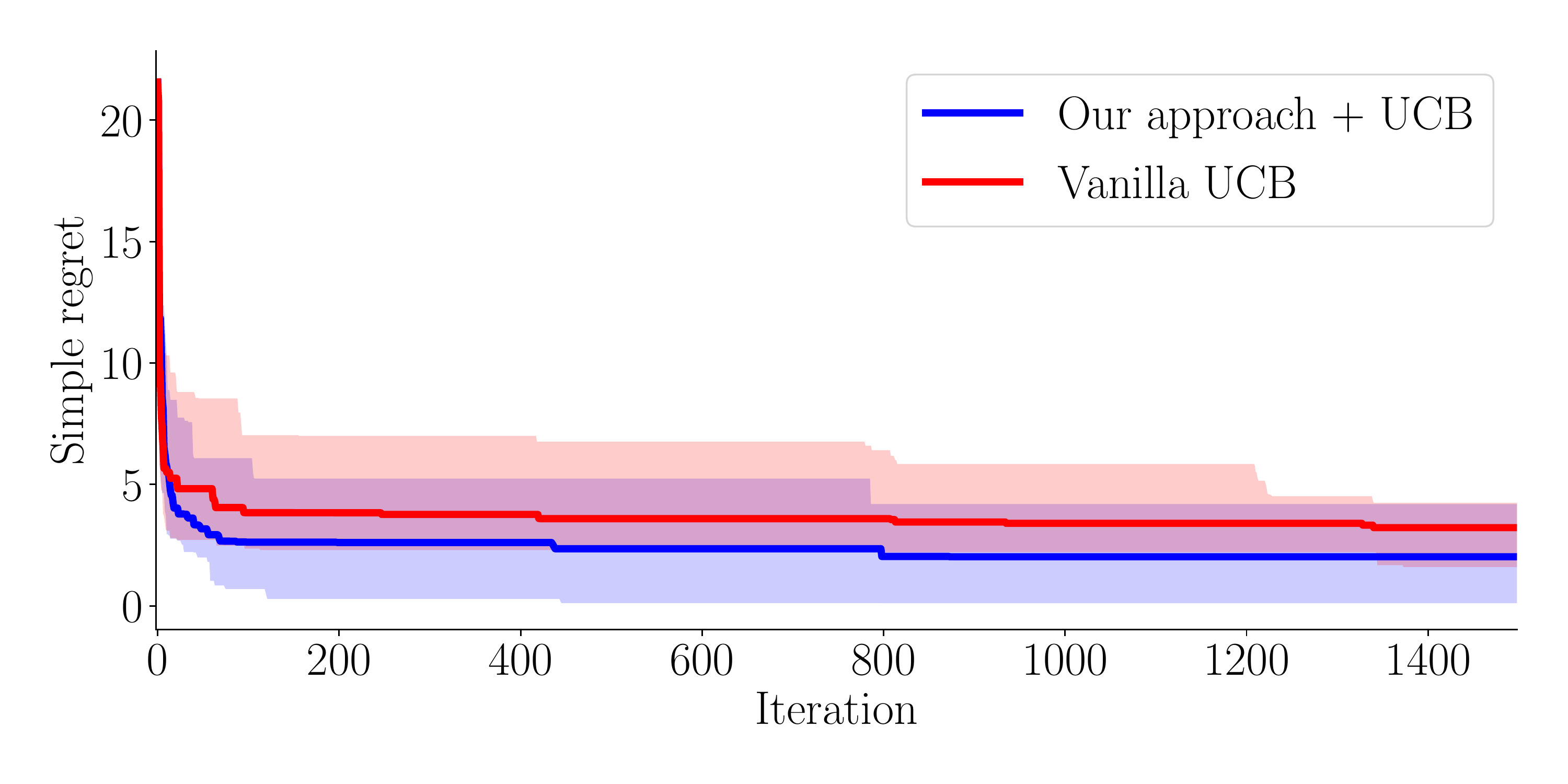}
\caption{Simple regret $R_T^{\text{simple}}$ of Ackley experiment.} 
\label{fig:ackleysimple}
\end{subfigure}
\begin{subfigure}[b]{0.495\textwidth}
\includegraphics[width = 0.995\columnwidth]{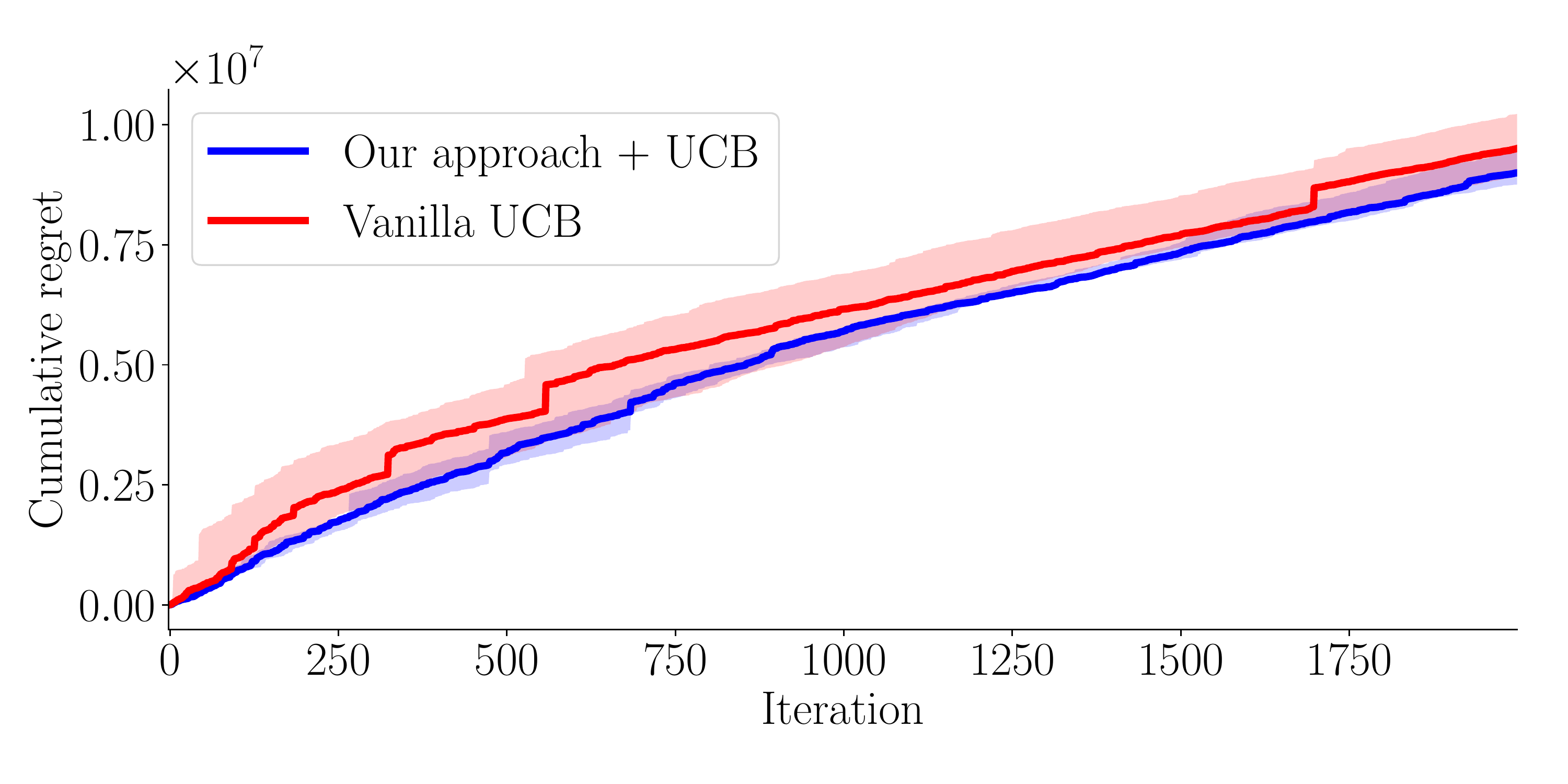}
\caption{Cumulative regret $R_T^{\text{cumul}}$ of Rosenbrock experiment.} 
\label{fig:rosenbrockcumulative}
\end{subfigure}
\hfill
\begin{subfigure}[b]{0.495\textwidth}
\includegraphics[width = 0.995\columnwidth]{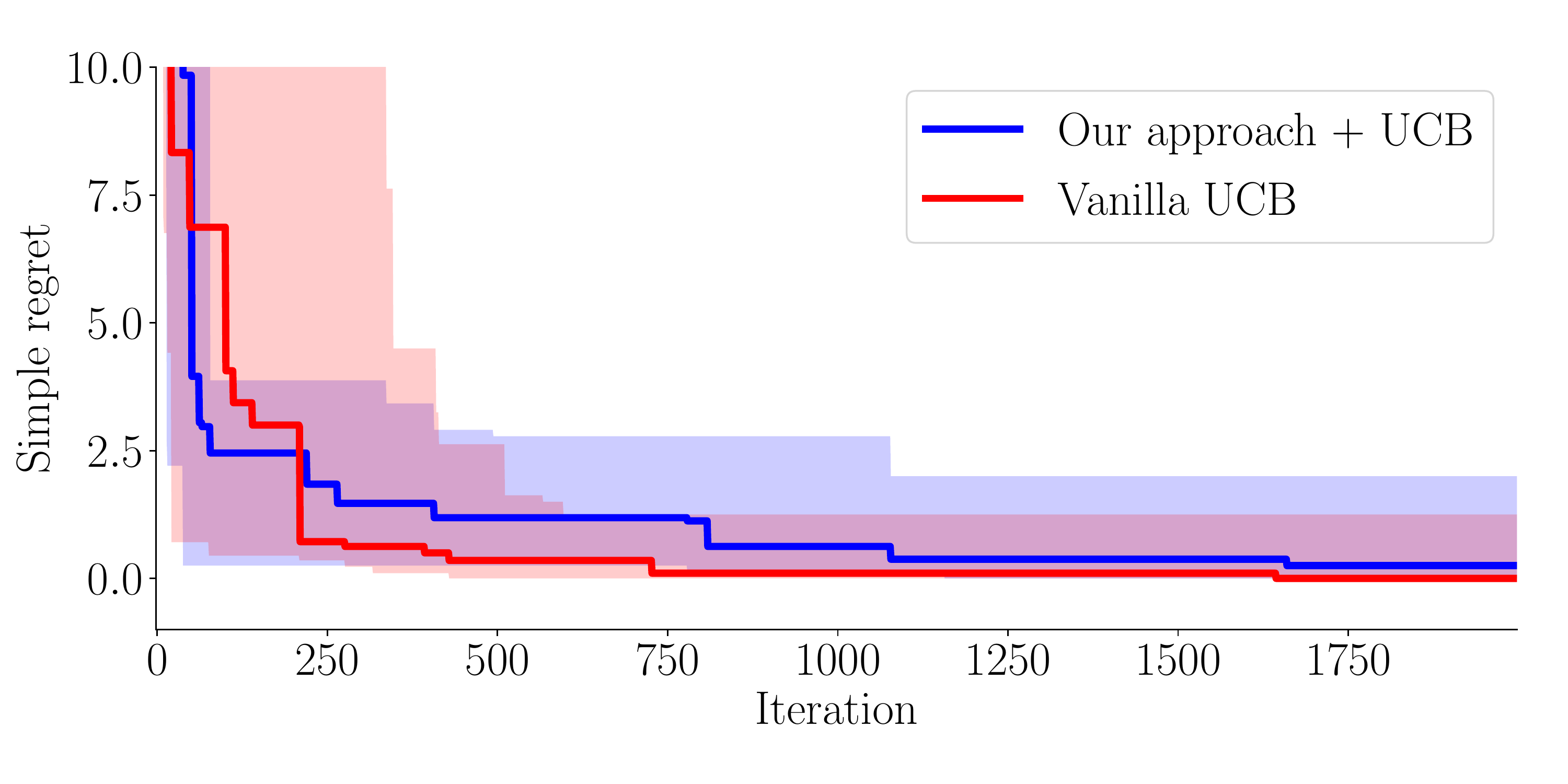}
\caption{Simple regret $R_T^{\text{simple}}$ of Rosenbrock experiment.} 
\label{fig:rosenbrocksimple}
\end{subfigure}
\caption{Regret of Ackley (top) and Rosenbrock (bottom) experiment over the number of Bayesian optimization iterations with UCB.}
\label{fig:regret}
\end{figure*}

%%%%%%%%%%%%%%%%%%%%%%%%%%%%%%%%%%%%%%%%%%%%%%%%%%%%%%%%%%%%

\end{document}